\documentclass{article}

\usepackage{microtype}
\usepackage{graphicx}
\usepackage{booktabs} 

\usepackage{hyperref}

\usepackage[accepted]{icml2023}

\usepackage{amsmath}
\usepackage{amssymb}
\usepackage{mathtools}
\usepackage{amsthm}
\usepackage{bm}

\usepackage[capitalize,noabbrev]{cleveref}

\usepackage{amsmath,amsfonts,bm}
\usepackage{mathtools}

\def\eqref#1{equation~\ref{#1}}

\def\1{\bm{1}}

\def\eps{{\epsilon}}

\def\vc{{\bm{c}}}

\def\vg{{\bm{g}}}
\def\vh{{\bm{h}}}

\def\vj{{\bm{j}}}

\def\vm{{\bm{m}}}

\def\vs{{\bm{s}}}
\def\vt{{\bm{t}}}

\def\vv{{\bm{v}}}
\def\vw{{\bm{w}}}
\def\vx{{\bm{x}}}

\def\mA{{\bm{A}}}

\def\mI{{\bm{I}}}

\def\mM{{\bm{M}}}

\def\mP{{\bm{P}}}
\def\mQ{{\bm{Q}}}

\def\mS{{\bm{S}}}

\def\mV{{\bm{V}}}

\def\mX{{\bm{X}}}

\DeclareMathAlphabet{\mathsfit}{\encodingdefault}{\sfdefault}{m}{sl}
\SetMathAlphabet{\mathsfit}{bold}{\encodingdefault}{\sfdefault}{bx}{n}

\def\gC{{\mathcal{C}}}

\def\gG{{\mathcal{G}}}
\def\gH{{\mathcal{H}}}

\def\gN{{\mathcal{N}}}
\def\gO{{\mathcal{O}}}

\def\gV{{\mathcal{V}}}

\def\gX{{\mathcal{X}}}

\def\sR{{\mathbb{R}}}

\DeclarePairedDelimiter{\norm}{\lVert}{\rVert}

\def\fG{{\mathfrak{G}}}
\def\fg{{\mathfrak{g}}}

\newcommand{\defeq}{\vcentcolon=}

\usepackage{soul, color}
\newcommand{\Note}[1]{}
\renewcommand{\Note}[1]{#1}  

\newcommand{\gray}[1]{\textcolor{gray}{#1}}

\usepackage{multirow}           
\usepackage{amsfonts}           
\usepackage{duckuments}         
\usepackage{thmtools,thm-restate}
\usepackage{enumitem}
\usepackage{todonotes}
\usepackage{xcolor,colortbl}
\usepackage{tikz}
\usepackage{tikz-cd}
\usepackage{caption}
\usepackage{subcaption}

\newtheorem{theorem}{Theorem}
\newtheorem{lemma}[theorem]{Lemma}

\newtheorem{definition}[theorem]{Definition}

\newcommand*{\ldblbrace}{\{\mskip-5mu\{}
\newcommand*{\rdblbrace}{\}\mskip-5mu\}}

\icmltitlerunning{On the Expressive Power of Geometric Graph Neural Networks}

\begin{document}

\twocolumn[
\icmltitle{On the Expressive Power of Geometric Graph Neural Networks}



\icmlsetsymbol{equal}{*}

\begin{icmlauthorlist}
\icmlauthor{Chaitanya K. Joshi}{equal,cambridge}
\icmlauthor{Cristian Bodnar}{equal,cambridge}
\icmlauthor{Simon V. Mathis}{cambridge}
\icmlauthor{Taco Cohen}{qualcomm}
\icmlauthor{Pietro Liò}{cambridge}
\end{icmlauthorlist}

\icmlaffiliation{cambridge}{University of Cambridge, UK}
\icmlaffiliation{qualcomm}{Qualcomm AI Research, The Netherlands. Qualcomm AI Research is an initiative of Qualcomm Technologies, Inc}

\icmlcorrespondingauthor{Chaitanya K. Joshi}{chaitanya.joshi@cl.cam.ac.uk}

\icmlkeywords{Geometric Deep Learning, Graph Neural Networks, Expressive Power, Equivariance, Graph Isomorphism}

\vskip 0.3in
]

\printAffiliationsAndNotice{\icmlEqualContribution} 

\begin{abstract}
    The expressive power of Graph Neural Networks (GNNs) has been studied extensively through the Weisfeiler-Leman (WL) graph isomorphism test. 
    However, standard GNNs and the WL framework are inapplicable for \emph{geometric graphs} embedded in Euclidean space, such as biomolecules, materials, and other physical systems.
    In this work, we propose a geometric version of the WL test (GWL) for discriminating geometric graphs while respecting the underlying physical symmetries: permutations, rotation, reflection, and translation.
    We use GWL to characterise the expressive power of geometric GNNs that are \emph{invariant} or \emph{equivariant} to physical symmetries in terms of distinguishing geometric graphs. 
    GWL unpacks how key design choices influence geometric GNN expressivity:
    (1) Invariant layers have limited expressivity as they cannot distinguish one-hop identical geometric graphs;
    (2) Equivariant layers distinguish a larger class of graphs by propagating geometric information beyond local neighbourhoods;
    (3) Higher order tensors and scalarisation enable maximally powerful geometric GNNs;
    and
    (4) GWL's discrimination-based perspective is equivalent to universal approximation.
    Synthetic experiments supplementing our results are available at \url{https://github.com/chaitjo/geometric-gnn-dojo}
\end{abstract}

\begin{figure}[h!]
    \centering
    \includegraphics[width=0.9\linewidth]{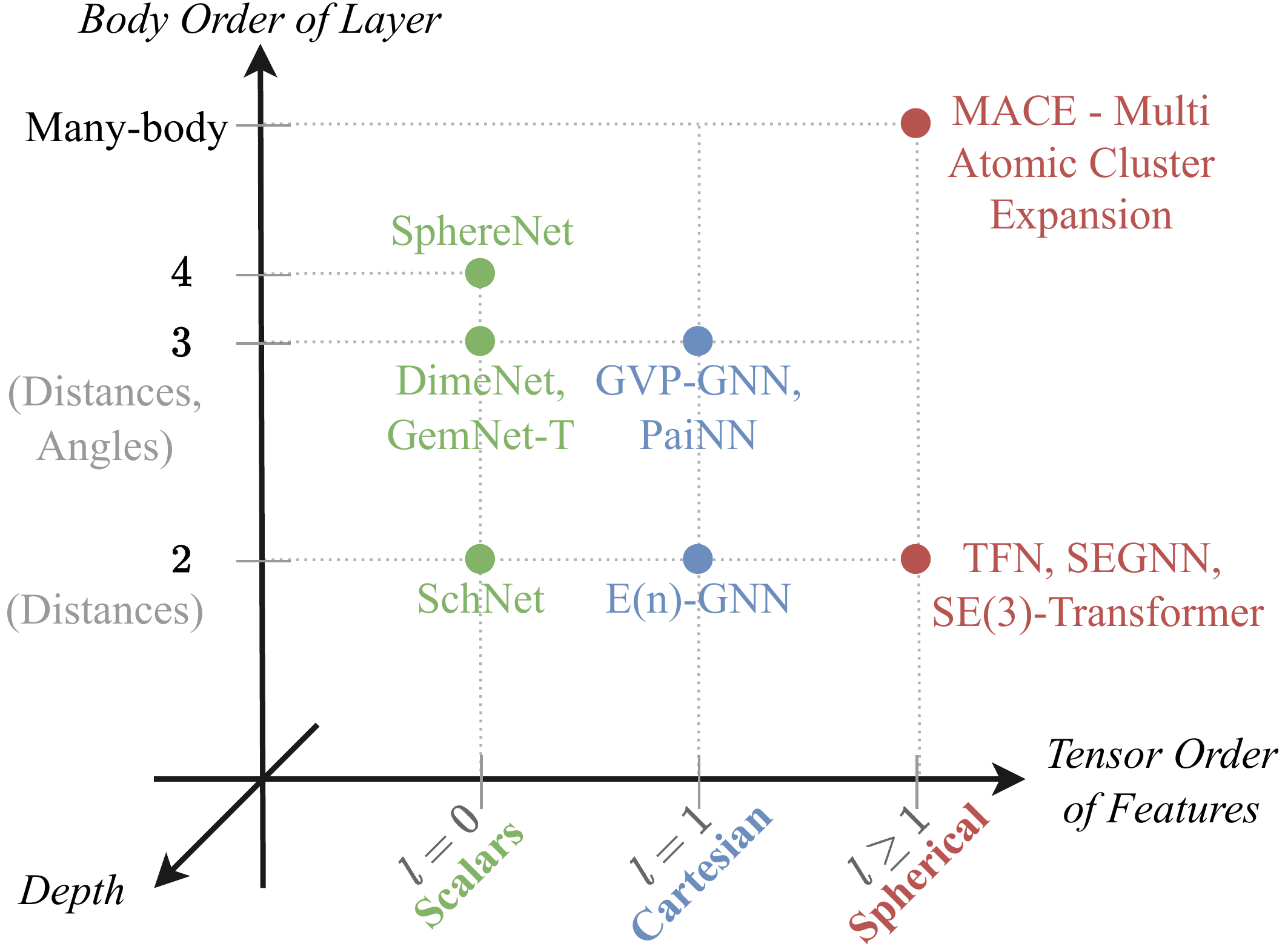}
    \caption{
    \textbf{Axes of geometric GNN expressivity:}
    (1) \emph{Body order}: increasing scalarisation body order builds expressive local descriptors; (2) \emph{Tensor order}: higher order tensors determine the relative orientation of neighbourhoods; and (3) \emph{Depth}: deep equivariant layers propagate geometric information beyond local neighbourhoods.
    }
    \label{fig:axes}
\end{figure}

\section{Introduction}

Systems in biochemistry \citep{jamasb2021graphein}, material science \citep{ocp_dataset}, physical simulations \citep{sanchez2020learning}, and multiagent robotics \citep{li2020graph} contain both geometry and relational structure. 
Such systems can be modelled via \emph{geometric graphs} embedded in Euclidean space.
For example, molecules are represented as a set of nodes which contain information about each atom and its 3D spatial coordinates as well as other geometric quantities such as velocity or acceleration.
Notably, the geometric attributes transform along with Euclidean transformations of the system, \textit{i.e.} they are equivariant to symmetry groups of rotations, reflections, and translation.
Standard Graph Neural Networks (GNNs) which do not take spatial symmetries into account are ill-suited for geometric graphs, as the geometric attributes would no longer retain their physical meaning and transformation behaviour \citep{bogatskiy2022symmetry,bronstein2021geometric}.

GNNs specialised for geometric graphs follow the message passing paradigm \citep{gilmer2017neural} where node features are updated in a permutation equivariant manner by aggregating features from local neighbourhoods.
Crucially, in addition to permutations, the updated geometric features of the nodes retain the transformation behaviour of the initial attributes, 
\textit{i.e.} they are also equivariant to the Lie group of rotations ($SO(d)$) or rotations and reflections ($O(d)$). We use $\fG$ as a generic symbol for these Lie groups.
We consider two classes of geometric GNNs:
(1) \textbf{$\fG$-equivariant models}, where the intermediate features and propagated messages are equivariant geometric quantities such as vectors \citep{satorras2021n} or tensors \citep{thomas2018tensor}; and
(2) \textbf{$\fG$-invariant models}, which only propagate local invariant scalar features such as distances and angles \citep{schutt2018schnet, Gasteiger2020directional}. 

Both classes of architectures have shown promising results for applications ranging from protein structure prediction \citep{baek2021accurate} and design \citep{dauparas2022robust} to molecular dynamics \citep{batzner2022nequip} and catalysis \citep{Gasteiger2021gemnet}.
At the same time, key theoretical questions remain unanswered:
(1) How to characterise the \textit{expressive power} of geometric GNNs? 
and
(2) What is the tradeoff between $\fG$-equivariant and $\fG$-invariant GNNs? 

The graph isomorphism problem~\citep{read1977graph} and the Weisfeiler-Leman (WL)~\citep{weisfeiler1968reduction} test for distinguishing non-isomorphic graphs have become a powerful tool for analysing the expressive power of non-geometric GNNs \citep{xu2018how, morris2019weisfeiler}.
The WL framework has been a major driver of progress in graph representation learning \citep{chen2019equivalence, maron2019provably, dwivedi2020benchmarking, bodnar2021swl, bodnar2021cwl}.
However, WL does not directly apply to geometric graphs as they exhibit a stronger notion of geometric isomorphism that accounts for spatial symmetries.

\textbf{Contributions. }
In this work, we study the expressive power of geometric GNNs from the perspective of discriminating non-isomorphic geometric graphs.
\vspace{-1em}
\begin{itemize}[itemsep=0em]
    \item In Section \ref{sec:gwl}, we propose a geometric version of the Weisfeiler-Leman graph isomorphism test, termed GWL, which is a theoretical upper bound on the expressive power of geometric GNNs.
    \item In Section \ref{sec:designspace}, we use the GWL framework to formalise how key design choices influence geometric GNN expressivity, summarised in Figure \ref{fig:axes}.
    Notably, invariant GNNs cannot distinguish graphs where one-hop neighbourhoods are the same and fail to compute non-local geometric properties such as volume, centroid, etc. 
    Equivariant GNNs distinguish a larger class of graphs as stacking equivariant layers propagates geometric information beyond local neighbourhoods.
    \item Synthetic experiments in Section \ref{sec:expt} demonstrate practical implications for building maximally powerful geometric GNNs, \textit{s.a.} geometric oversquashing with increased depth, and counterexamples that highlight the utility of higher order spherical tensors.
    \item Finally, in Section \ref{sec:universality}, we prove an equivalence between a model's ability to discriminate geometric graphs and its universal approximation capabilities.
    While universality is binary, GWL's discrimination-based perspective provides a more granular and practically insightful lens into geometric GNN expressivity.
\end{itemize}

\section{Background}
\label{sec:prelim}

\textbf{Graph Isomorphism. }
An attributed graph $\mathcal{G} = ( \mA, \mS )$ is a set $\mathcal{V}$ of $n$ nodes connected by edges.
$\mA$ denotes an $n \times n$ adjacency matrix where each entry $a_{ij} \in \{ 0, 1 \}$ indicates the presence or absence of an edge connecting nodes $i$ and $j$.
The matrix of \emph{scalar} features $\mS \in \mathbb{R}^{n \times f}$ stores attributes $\vs_i \in \mathbb{R}^f$ associated with each node $i$, where the subscripts index rows and columns in the corresponding matrices.

The graph isomorphism problem asks whether two graphs are the same, but drawn differently.
Two attributed graphs $\gG, \gH$ are \emph{isomorphic} (denoted $\mathcal{G} \simeq \mathcal{H}$)  if there exists an edge-preserving bijection $b: \mathcal{V}(\mathcal{G}) \rightarrow \mathcal{V}(\mathcal{H})$ such that $\vs_i^{(\mathcal{G})} = \vs_{b(i)}^{(\mathcal{H})}$.

\textbf{Weisfeiler-Leman Test (WL). }
WL tests whether two (attributed) graphs are isomorphic. At iteration zero, WL assigns a \emph{colour} $c_i^{(0)} \in C$ from a countable space of colours $C$ to each node $i$. 
Nodes are coloured the same if their features are the same, otherwise, they are coloured differently. 
In subsequent iterations $t$, WL iteratively updates the colouring $\vc_i^{(t)} \in C$:
\vspace{-0.5em}
\begin{equation}\label{eq:wl}
    \vc^{(t)}_i \defeq \textsc{Hash} \left( \vc^{(t-1)}_i , \ldblbrace \vc^{(t-1)}_j \mid j \in \mathcal{N}_i \rdblbrace \right),
\end{equation}
where $\textsc{Hash}$ is an injective map (\textit{i.e.} a perfect hash map) that assigns a unique colour to each input.
The test terminates when the partition of the nodes induced by the colours becomes stable.
Given two graphs $\mathcal{G}$ and $\mathcal{H}$, if there exists some iteration $t$ for which $\ldblbrace c_i^{(t)} \mid i \in \mathcal{V}(\mathcal{G}) \rdblbrace \neq \ldblbrace c_i^{(t)} \mid i \in \mathcal{V}(\mathcal{H}) \rdblbrace$, then the graphs are not isomorphic. 
Otherwise, WL is inconclusive and we cannot distinguish the two graphs.


\textbf{Group Theory. } 
We assume basic familiarity with group theory, see \citet{zee2016group} for an overview.
We denote the action of a group $\fG$ on a space $X$ by $\fg \cdot x$. 
If $\fG$ acts on spaces $X$ and $Y$, we say:
(1) A function $f: X \to Y$ is $\fG$-\emph{invariant} if $f(\fg \cdot x) = f(x)$, \textit{i.e.} the output remains unchanged under transformations of the input; and
(2) A function $f: X \to Y$ is $\fG$-\emph{equivariant} if $f(\fg \cdot x) = \fg \cdot f(x)$, \textit{i.e.} a transformation of the input must result in the output transforming equivalently.

We also consider \emph{$\fG$-orbit injective} functions:
The \emph{$\fG$-orbit} of $x \in X$ is $\gO_{\fG}(x) = \{ \fg \cdot x \mid \fg \in \fG \} \subseteq X$.
When $x$ and $x'$ are part of the same orbit, we write $x \simeq x'$. 
We say a function $f: X \to Y$ is \emph{$\fG$-orbit injective} if we have $f(x_1) = f(x_2)$ if and only if  $x_1 \simeq x_2$ for any $x_1, x_2 \in X$.
Such a function is $\fG$-invariant, since $f(\fg \cdot x) = f(x)$.

We work with the permutation group over $n$ elements $S_n$ and the Lie groups $\fG = SO(d)$ or $\fG = O(d)$.
Invariance to translations $T(d)$ is conventionally handled using relative positions. 
Given one of the groups above, for an element $\fg$ we denote by $\mM_\fg$ (or another capital letter) its standard matrix representation.


\textbf{Geometric Graphs. }
A geometric graph $\mathcal{G} = ( \mA, \mS, \vec{\mV}, \vec{\mX} )$ is an attributed graph that is also decorated with geometric attributes: node coordinates $\vec{\mX} \in \mathbb{R}^{n \times d}$ and (optionally) vector features
$\vec{\mV} \in \mathbb{R}^{n \times d}$ (\textit{e.g.} velocity, acceleration).
Without loss of generality, we work with a single vector feature per node. Our setup generalises to multiple vector features or higher-order tensors.
The geometric attributes transform as follows under the action of the relevant groups:
(1) $S_n$ acts on the graph via $\mP_\sigma \gG := (\mP_\sigma\mA\mP_\sigma^\top, \mP_\sigma\mS, \mP_\sigma\vec{\mV}, \mP_\sigma\vec{\mX})$;
(2) Orthogonal transformations $\mQ_\fg \in \fG$ act on $\vec{\mV}, \vec{\mX}$ via $\vec{\mV}\mQ_\fg, \vec{\mX}\mQ_\fg$;
and
(3) Translations $\vec{\vt} \in T(d)$ act on the coordinates $\vec{\mX}$ via $\vec{\vx}_i + \vec{\vt}$ for all nodes $i$.

Two geometric graphs $\mathcal{G}$ and $\mathcal{H}$ are \emph{geometrically isomorphic} if there exists an attributed graph isomorphism $b$ such that the geometric attributes are equivalent, up to global group actions $\mQ_\fg \in \fG$ and $\vec{\vt} \in T(d)$:
\vspace{-0.5em}
\begin{multline}
    \left( \vs_{b(i)}^{(\mathcal{H})}, \ \mQ_{\fg} \vec{\vv}_{b(i)}^{(\mathcal{H})}, \ \mQ_{\fg} (\vec{\vx}_{b(i)}^{(\mathcal{H})} + \vec{\vt}) \right) \ = \ \\
    \left( \vs_i^{(\mathcal{G})}, \vec{\vv}_i^{(\mathcal{G})}, \vec{\vx}_i^{(\mathcal{G})} \right) \quad \text{for all } i \in \mathcal{V}(\mathcal{G}).
\end{multline}
Geometric graph isomorphism and distinguishing (sub-) graph geometries has important practical implications for representation learning. 
For \textit{e.g.}, for molecular systems, an ideal architecture should map distinct local structural environments around atoms to distinct representations~\citep{bartok2013representing, pozdnyakov2020incompleteness}, which can then be used for downstream predictive task.


\textbf{Geometric Graph Neural Networks. }
We consider two broad classes of geometric GNN architectures.
(1) \emph{$\fG$-equivariant GNN layers} \citep{thomas2018tensor, satorras2021n} 
propagate scalar as well as geometric vector features from iteration $t$ to $t+1$ via learnable aggregate and update functions, $\textsc{Agg}$ and $\textsc{Upd}$, respectively:
\vspace{-0.5em}
\begin{multline}
    \label{eq:gnn-equiv}
    \vs_i^{(t+1)}, \vec{\vv}_i^{(t+1)} \defeq \textsc{Upd} \Big( (\vs_i^{(t)}, \vec{\vv}_i^{(t)}) \ , \\
    \textsc{Agg} \left( \ldblbrace (\vs_i^{(t)}, \vec{\vv}_i^{(t)}), (\vs_j^{(t)},  \vec{\vv}_j^{(t)}), \vec{\vx}_{ij} \mid j \in \mathcal{N}_i \rdblbrace \right) \Big),
\end{multline}
where $\vec{\vx}_{ij} = \vec{\vx}_{i} - \vec{\vx}_{j}$ denote relative position vectors.
Alternatively, (2) \emph{$\fG$-invariant GNN layers}
\citep{schutt2018schnet, Gasteiger2020directional} 
do not update vector features and only propagate invariant scalar quantities computed using geometric information within local neighbourhoods:
\vspace{-0.5em}
\begin{multline}
    \label{eq:gnn-inv}
    \vs_i^{(t+1)} \defeq \textsc{Upd} \Big( \vs_i^{(t)} \ , \\ 
    \textsc{Agg} \left( \ldblbrace (\vs_i^{(t)}, \vec{\vv}_i^{(t)}), (\vs_j^{(t)},  \vec{\vv}_j^{(t)}), \vec{\vx}_{ij} \mid j \in \mathcal{N}_i \rdblbrace \right) \Big).
\end{multline}
For both models, the scalar features at the final iteration are mapped to graph-level predictions via a permutation-invariant readout $f: \mathbb{R}^{n \times f} \rightarrow \mathbb{R}^{f'}$.
\textbf{In Appendix \ref{app:gnn-bg}, we provide unified equations and examples of geometric GNNs covered by our framework.}


\section{The Geometric Weisfeiler-Leman Test}
\label{sec:gwl}


\begin{figure*}[t]
    \centering
    \includegraphics[width=0.75\linewidth]{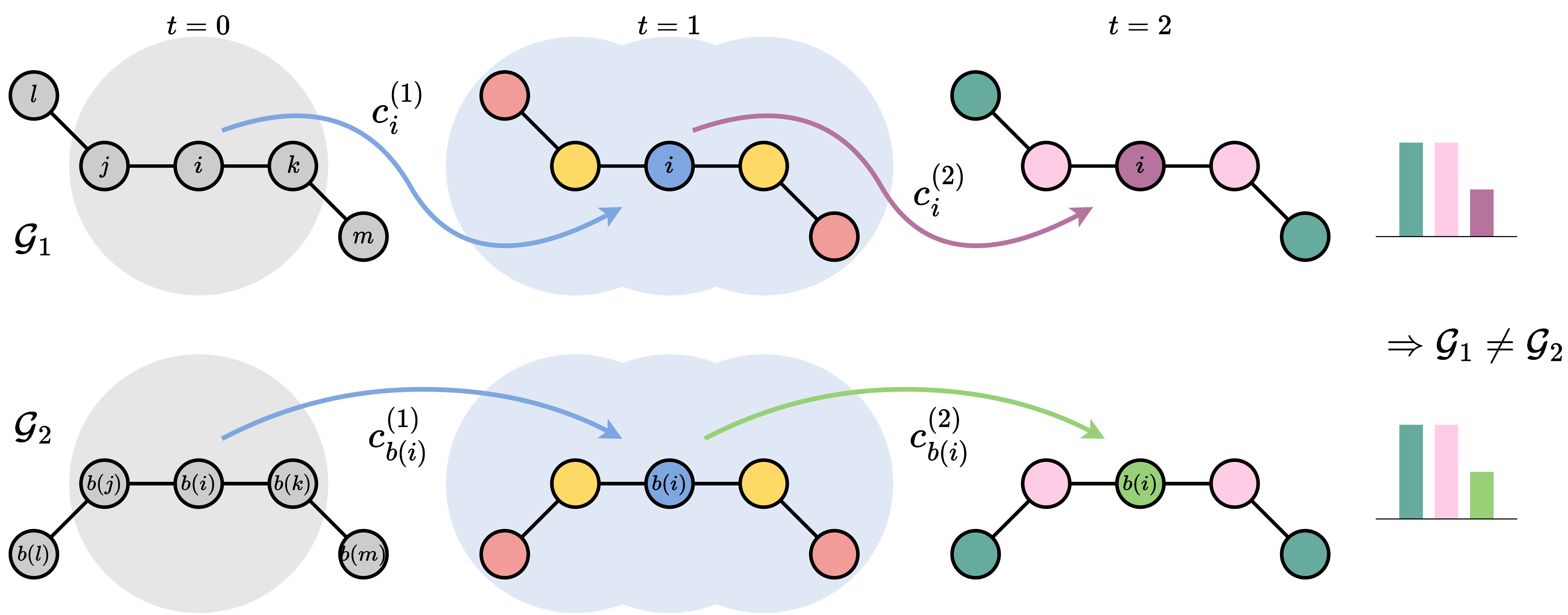}
    \caption{\textbf{Geometric Weisfeiler-Leman Test.}
        GWL distinguishes non-isomorphic geometric graphs $\mathcal{G}_1$ and $\mathcal{G}_2$ by injectively assigning colours to distinct neighbourhood patterns, up to global symmetries.
        Each iteration expands the neighbourhood from which geometric information can be gathered (shaded for node $i$). 
        Example inspired by \citet{schutt2021equivariantmp}, $\fG = O(d)$.
    }
    \label{fig:gwl}
\end{figure*}


\textbf{Assumptions.} Analogous to the WL test,
we assume for the rest of this section that all geometric graphs we want to distinguish are finite in size and come from a countable dataset. In other words, 
the geometric and scalar features the nodes are equipped with come from countable subsets $C \subset \sR^d$ and $C' \subset \sR$, respectively. 
As a result, when we require functions to be injective, we require them to be injective over the countable set of $\fG$-orbits that are obtained by acting with $\fG$ on the dataset.

\textbf{Intuition.}
For an intuition of how to generalise WL to geometric graphs, we note that WL uses a local, node-centric, procedure to update the colour of each node $i$ using the colours of its the 1-hop \emph{neighbourhood} $\mathcal{N}_i$.
In the geometric setting, $\mathcal{N}_i$ is an attributed point cloud around the central node $i$. As a result, each neighbourhood carries two types of information: (1) neighbourhood type (invariant to $\fG$) and (2) neighbourhood geometric orientation (equivariant to $\fG$).
This leads to constraints on the generalisation of the neighbourhood aggregation procedure of Equation~\ref{eq:wl}. 
From an axiomatic point of view, our generalisation of the WL aggregation procedure must meet two properties:

\textbf{Property 1: Orbit injectivity of colours.}
If two neighbourhoods are the same up to an action of $\fG$ (\textit{e.g.} rotation), then the colours of the corresponding central nodes should be the same.
Thus, the colouring must be $\fG$-orbit injective -- which also makes it $\fG$-invariant -- over the countable set of all orbits of neighbourhoods in our dataset.

\textbf{Property 2: Preservation of local geometry.} A key property of WL is that the aggregation is injective.
A $\fG$-invariant colouring procedure that purely satisfies Property 1 is not sufficient because, by definition, it loses spatial properties of each neighbourhood such as the relative pose or orientation \citep{hinton2011transforming}.
Thus, we must additionally update auxiliary \emph{geometric information} variables in a way that is $\fG$-equivariant and injective.

\textbf{Geometric Weisfeiler-Leman (GWL).} These intuitions motivate the following definition of the GWL test.
At initialisation, we assign to each node $i \in \mathcal{V}$ a scalar node colour $c_i \in C'$ and an auxiliary object $\vg_i$ containing the geometric information associated to it:
\vspace{-0.5em}
\begin{align}
    \label{eq:gwl_init}
    c_i^{(0)} \defeq \textsc{Hash}(\vs_i), \quad {\vg}_i^{(0)} \defeq \left( c_i^{(0)}, \vec{\vv}_i
    \right),
\end{align}
where $\textsc{Hash}$ denotes an injective map over the scalar attributes $\vs_i$ of node $i$, \textit{e.g.} the $\textsc{Hash}$ of standard WL. 
To define the inductive step, assume we have the colours of the nodes and the associated geometric objects at iteration $t-1$. Then, we can aggregate the geometric information around node $i$ into a new object as follows:
\vspace{-0.5em}
\begin{multline}\label{eq:gwl_geom}
    \vg_i^{(t)} \defeq \big( (c_i^{(t-1)}, \vg_i^{(t-1)}) \ , \\  \ldblbrace  (c_j^{(t-1)}, \vg_j^{(t-1)}, \vec{\vx}_{ij}) \mid j \in \mathcal{N}_i \rdblbrace \big),
\end{multline}
Here $\ldblbrace \cdot \rdblbrace$ denotes a multiset -- a set in which elements may occur more than once. 
Importantly, the group $\fG$ can act on the geometric objects above inductively by acting on the geometric information inside it. This amounts to rotating (or reflecting) the entire $t$-hop neighbourhood contained inside:
\vspace{-0.5em}
\begin{align}
    \fg \cdot {\vg}_i^{(0)} &\defeq \left( c_i^{(0)}, \ \mQ_{\fg} \vec{\vv}_i \right), \\
    \fg \cdot \vg_i^{(t)} &\defeq \Big( (c_i^{(t-1)}, \fg \cdot  \vg_i^{(t-1)}), \\ & \ldblbrace  (c_j^{(t-1)}, \fg \cdot \vg_j^{(t-1)}, \mQ_{\fg} \vec{\vx}_{ij}) \mid j \in \mathcal{N}_i \rdblbrace \Big) \nonumber
\end{align}
Clearly, the aggregation building $\vg_i$ for any time-step $t$ is injective and $\fG$-equivariant.
Finally, we can compute the node colours at iteration $t$ for all $i \in \mathcal{V}$ by aggregating the geometric information in the neighbourhood around node $i$:
\vspace{-0.5em}
\begin{align}
    \label{eq:gwl_hash}
    c_i^{(t)} & \defeq \textsc{I-Hash}^{(t)} \left( \vg_i^{(t)} \right),
\end{align}
by using a $\fG$-orbit injective and $\fG$-invariant function that we denote by $\textsc{I-Hash}$. That is for any geometric objects $\vg, \vg'$, $\textsc{I-Hash} (\vg) = \textsc{I-Hash} (\vg')$ if and only if there exists $\fg \in \fG$ such that $\vg = \fg \cdot \vg'$.
Note that $\textsc{I-Hash}$ is an idealised $\fG$-orbit injective function, similar to the $\textsc{Hash}$ function used in WL, which is not necessarily continuous. 

\textbf{Overview. } With each iteration, $\vg_i^{(t)}$ aggregates geometric information in progressively larger $t$-hop subgraph neighbourhoods $\gN_i^{(t)}$ around the node $i$.
The node colours summarise the structure of these $t$-hops via the $\fG$-invariant aggregation performed by $\textsc{I-Hash}$. The procedure terminates when the partitions of the nodes induced by the colours do not change from the previous iteration. Finally, given two geometric graphs $\mathcal{G}$ and $\mathcal{H}$, if there exists some iteration $t$ for which $\ldblbrace c_i^{(t)} \mid i \in \mathcal{V}(\mathcal{G}) \rdblbrace \neq \ldblbrace c_i^{(t)} \mid i \in \mathcal{V}(\mathcal{H}) \rdblbrace$, then GWL deems the two graphs as geometrically non-isomorphic. Otherwise, GWL cannot distinguish them.




\textbf{Invariant GWL. }
Since we are interested in understanding the role of $\fG$-equivariance, we also consider a more restrictive Invariant GWL (IGWL) that only updates node colours using the $\fG$-orbit injective $\textsc{I-Hash}$ function and does not propagate geometric information:
\vspace{-0.5em}
\begin{multline}
    c_i^{(t)} \defeq \textsc{I-Hash} \Big( (c_i^{(t-1)} , \vec{\vv}_i) \ ,  \\ 
    \ldblbrace (c_j^{(t-1)} , \vec{\vv}_j, \ \vec{\vx}_{ij} ) \mid j \in \mathcal{N}_i \rdblbrace \Big). \label{eq:inv_gwl}
\end{multline}


\textbf{IGWL with $k$-body scalars. }
In order to further analyse the construction of the node colouring function $\textsc{I-Hash}$, we consider $\text{IGWL}_{(k)}$ based on the maximum number of nodes involved in the computation of $\fG$-invariant scalars (also known as the `body order' \citep{batatia2022design}):
\vspace{-0.5em}
\begin{multline}
    c_i^{(t)} \defeq \textsc{I-Hash}_{(k)} \Big( (c_i^{(t-1)} , \vec{\vv}_i) \ ,  \\ \ldblbrace (c_j^{(t-1)} , \vec{\vv}_j, \ \vec{\vx}_{ij} ) \mid j \in \mathcal{N}_i \rdblbrace \Big),
\end{multline}
and $\textsc{I-Hash}_{(k+1)}$ is defined as:
\vspace{-0.5em}
\begin{multline}
    \textsc{Hash} \Big( \ldblbrace \textsc{I-Hash} \Big( (c_i^{(t-1)} , \vec{\vv}_i),
    \ldblbrace (c_{j_1}^{(t-1)} , \vec{\vv}_{j_1}, \ \vec{\vx}_{ij_1} ),\dots, \\ (c_{j_k}^{(t-1)} , \vec{\vv}_{j_k}, \ \vec{\vx}_{ij_k} ) \rdblbrace \Big)
    \mid \vj \in (\mathcal{N}_i)^k 
    \rdblbrace \Big),
\end{multline}
where $\vj = [j_1, \dots, j_k ]$ are all possible $k$-tuples formed of elements of $\mathcal{N}_i$. 
Therefore, $\text{IGWL}_{(k)}$ is now constrained to extract information only from all the possible $k$-sized tuples of nodes (including the central node) in a neighbourhood.
For instance, $\textsc{I-Hash}_{(2)}$ can identify neighbourhoods only up to pairwise distances among the central node and any of its neighbours (\textit{i.e.} a $2$-body scalar), while $\textsc{I-Hash}_{(3)}$ up to distances and angles formed by any two edges (\textit{i.e.} a $3$-body scalar).
Notably, distances and angles alone are incomplete descriptors of local geometry as there exist several counterexamples that require higher body-order invariants to be distinguished 
\citep{bartok2013representing, pozdnyakov2020incompleteness}. Therefore, $\textsc{I-Hash}_{(k)}$ with lower scalarisation body order $k$ makes the colouring weaker.


\subsection{Characterising the expressivity of geometric GNNs}
\label{sec:gwl:equivalence}

Following \citet{xu2018how,morris2019weisfeiler}, we upper bound the expressive power of geometric GNNs by the GWL test.
We show that any $\fG$-equivariant GNN can be at most as powerful as GWL in distinguishing non-isomorphic geometric graphs.
Proofs are available in Appendix \ref{app:gwl-equivalence}.

\begin{restatable}{theorem}{uppertheo}
    \label{theo:upper}
    Any pair of geometric graphs distinguishable by a $\fG$-equivariant GNN is also distinguishable by GWL.
\end{restatable}

With sufficient iterations, the output of $\fG$-equivariant GNNs can be equivalent to GWL if certain conditions are met regarding the aggregate, update and readout functions.

\begin{restatable}{proposition}{conditionstheo}
    \label{theo:conditions}
    $\fG$-equivariant GNNs have the same expressive power as GWL if the following conditions hold:
    (1) The aggregation $\textsc{Agg}$ is an injective, $\fG$-equivariant multiset function.
    (2) The scalar part of the update $\textsc{Upd}_s$ is a $\fG$-orbit injective, $\fG$-invariant multiset function.
    (3) The vector part of the update $\textsc{Upd}_v$ is an injective, $\fG$-equivariant multiset function.
    (4) The graph-level readout $f$ is an injective multiset function.
\end{restatable}


Similar statements can be made for $\fG$-invariant GNNs and IGWL.
Thus, we can directly transfer theoretical results between GWL/IGWL, which are abstract mathematical tools, and the class of geometric GNNs upper bounded by the respective tests.
This connection has several interesting practical implications, discussed subsequently.


\section{Understanding the Design Space of Geometric GNNs via GWL}
\label{sec:designspace}

\textbf{Overview. } 
We demonstrate the utility of the GWL framework for understanding how geometric GNN design choices \citep{batatia2022design} influence expressivity:
(1) Depth or number of layers; 
(2) Invariant vs. equivariant message passing;
and
(3) Body order of scalarisation.
In doing so, we formalise theoretical limitations of current architectures and provide practical implications for designing maximally powerful models.
Proofs are available in Appendix \ref{app:design-space}.

\subsection{Role of depth: propagating geometric information}
\label{sec:designspace:depth}


Let us consider the simplified setting of two geometric graphs $\mathcal{G}_1 = ( \mA_1, \mS_1, \vec{\mV}_1, \vec{\mX}_1 )$ and $\mathcal{G}_2 = ( \mA_2, \mS_2, \vec{\mV}_2, \vec{\mX}_2 )$ such that the underlying attributed graphs $( \mA_1, \mS_1 )$ and $( \mA_2, \mS_2 )$ are isomorphic.
This case frequently occurs in chemistry, where molecules occur in different conformations, but with the same graph topology given by the covalent bonding structure.

Each iteration of GWL aggregates geometric information $\vg_i^{(k)}$ from progressively larger neighbourhoods $\gN_i^{(k)}$ around the node $i$, and distinguishes (sub-)graphs via comparing $\fG$-orbit injective colouring of $\vg_i^{(k)}$.
We say $\mathcal{G}_1$ and $\mathcal{G}_2$ are \textit{$k$-hop distinct} if for all graph isomorphisms $b$, there is some node $i \in \mathcal{V}_1, b(i) \in \mathcal{V}_2$ such that the corresponding $k$-hop subgraphs $\gN_i^{(k)}$ and $\gN_{b(i)}^{(k)}$ are distinct.
Otherwise, we say $\mathcal{G}_1$ and $\mathcal{G}_2$ are \textit{$k$-hop identical} if all corresponding $k$-hop subgraphs $\gN_i^{(k)}$ and $\gN_{b(i)}^{(k)}$ are the same up to group  actions.
We can now formalise what geometric graphs can and cannot be distinguished by GWL and maximally powerful geometric GNNs, in terms of the number of iterations.

\begin{restatable}{proposition}{gwlcanprop}
    \label{prop:gwl-can}
    GWL can distinguish any $k$-hop distinct geometric graphs $\mathcal{G}_1$ and $\mathcal{G}_2$ where the underlying attributed graphs are isomorphic, and $k$ iterations are sufficient.
\end{restatable}

\begin{restatable}{proposition}{gwlcannotprop}
    \label{prop:gwl-cannot}
    Up to $k$ iterations, GWL cannot distinguish any $k$-hop identical geometric graphs $\mathcal{G}_1$ and $\mathcal{G}_2$ where the underlying attributed graphs are isomorphic.
\end{restatable}

Additionally, we can state the following results about the more constrained IGWL.

\begin{restatable}{proposition}{igwlcanprop}
    \label{prop:igwl-can}
    IGWL can distinguish any $1$-hop distinct geometric graphs $\mathcal{G}_1$ and $\mathcal{G}_2$ where the underlying attributed graphs are isomorphic,
    and 1 iteration is sufficient.
\end{restatable}

\begin{restatable}{proposition}{igwlcannotprop}
    \label{prop:igwl-cannot}
    Any number of iterations of IGWL cannot distinguish any $1$-hop identical geometric graphs $\mathcal{G}_1$ and $\mathcal{G}_2$ where the underlying attributed graphs are isomorphic.
\end{restatable}

An example illustrating Propositions \ref{prop:gwl-can} and \ref{prop:igwl-cannot} is shown in Figures \ref{fig:gwl} and \ref{fig:igwl}, respectively.

We can now consider the more general case where the underlying attributed graphs for $\mathcal{G}_1 = ( \mA_1, \mS_1, \vec{\mV}_1, \vec{\mX}_1 )$ and $\mathcal{G}_2 = ( \mA_2, \mS_2, \vec{\mV}_2, \vec{\mX}_2 )$ are non-isomorphic and constructed from point clouds using radial cutoffs, as conventional for biochemistry and material science applications.

\begin{restatable}{proposition}{nonisoprop}
    \label{prop:noniso}
    Assuming geometric graphs are constructed from point clouds using radial cutoffs,
    GWL can distinguish any geometric graphs $\mathcal{G}_1$ and $\mathcal{G}_2$ where the underlying attributed graphs are non-isomorphic.
    At most $k_{\text{Max}}$ iterations are sufficient, where $k_{\text{Max}}$ is the maximum graph diameter among $\mathcal{G}_1$ and $\mathcal{G}_2$.
\end{restatable}


\subsection{Limitations of invariant message passing: failure to capture global geometry}
\label{sec:designspace:inv-equiv}

Propositions \ref{prop:gwl-can} and \ref{prop:igwl-cannot} enable us to compare the expressive powers of GWL and IGWL.

\begin{restatable}{theorem}{stricttheo}
    \label{theo:strict}
    GWL is strictly more powerful than IGWL.
\end{restatable}

This statement formalises the advantage of $\fG$-equivariant intermediate layers for geometric graphs, as prescribed in the Geometric Deep Learning blueprint \citep{bronstein2021geometric}, in addition to echoing similar intuitions in the computer vision community.
As remarked by \citet{hinton2011transforming}, translation invariant models do not understand the relationship between the various parts of an image (termed the ``Picasso problem''). 
Similarly, our results point to IGWL and $\fG$-invariant GNNs failing to understand how the 1-hop neighbourhoods in a graph are oriented w.r.t. each other. 

As a result, even the most powerful $\fG$-invariant GNNs are restricted in their ability to compute global and non-local geometric properties.
\begin{restatable}{proposition}{globalgeomprop}
\label{prop:globalgeom}
IGWL and $\fG$-invariant GNNs cannot decide several geometric graph properties: (1) perimeter, surface area, and volume of the bounding box/sphere enclosing the geometric graph; (2) distance from the centroid or centre of mass; and (3) dihedral angles. 
\end{restatable}

Finally, we identify a setting where this distinction between the two approaches disappears.

\begin{restatable}{proposition}{limitoffullprop}
    \label{prop:limit-of-full}
    IGWL has the same expressive power as GWL for fully connected geometric graphs.
\end{restatable}


\textbf{Practical Implications. }
Proposition \ref{prop:globalgeom} highlights critical theoretical limitations of $\fG$-invariant GNNs.
This suggests that $\fG$-equivariant GNNs should be preferred when working with large geometric graphs such as macromolecules with thousands of nodes, where message passing is restricted to local radial neighbourhoods around each node.
Stacking multiple $\fG$-equivariant layers enables the computation of compositional geometric features.

Two straightforward approaches to overcoming the limitations of $\fG$-invariant GNNs may be:
(1) pre-computing non-local geometric properties as input features, \textit{e.g.} models such as GemNet \citep{Gasteiger2021gemnet} and ComENet \citep{wang2022comenet} use two-hop dihedral angles.
And
(2) working with fully connected geometric graphs, as Proposition \ref{prop:limit-of-full} suggests that $\fG$-equivariant and $\fG$-invariant GNNs can be made equally powerful when performing all-to-all message passing.
This is supported by the empirical success of recent $\fG$-invariant \emph{Graph Transformers} \citep{joshi2020transformers, shi2022benchmarking} for small molecules with tens of nodes, where working with full graphs is tractable.



\subsection{Role of scalarisation body order: identifying neighbourhood $\fG$-orbits}
\label{sec:designspace:scalarisation}

At each iteration of GWL and IGWL, the $\textsc{I-Hash}$ function assigns a $\fG$-invariant colouring to distinct geometric neighbourhood patterns.
$\textsc{I-Hash}$ is an idealised $\fG$-orbit injective function which is not necessarily continuous.
In geometric GNNs, this corresponds to scalarising local geometric information when updating the scalar features; examples are shown in \eqref{eq:schnet} and \eqref{eq:dimenet}.
We can analyse the construction of the $\textsc{I-Hash}$ function and the scalarisation step in geometric GNNs via the $k$-body variations $\text{IGWL}_{(k)}$, described in Section \ref{sec:gwl}.
In doing so, we will make connections between IGWL and WL for non-geometric graphs.

Firstly, we formalise the relationship between the injectivity of $\textsc{I-Hash}_{(k)}$ and the maximum cardinality of local neighbourhoods in a given dataset.

\begin{restatable}{proposition}{bodyorderprop}
    \label{prop:bodyorder}
    $\textsc{I-Hash}_{(m)}$ is $\fG$-orbit injective for $m = \text{max}(\lbrace \vert \gN_i \vert \mid i \in \mathcal{V} \rbrace)$, the maximum cardinality of all local neighbourhoods $\gN_i$ in a given dataset.
\end{restatable}

\textbf{Practical Implications. }
While building provably injective $\textsc{I-HASH}_{(k)}$ functions may require intractably high $k$, the hierarchy of $\text{IGWL}_{(k)}$ tests enable us to study the expressive power of practical $\fG$-invariant aggregators used in current geometric GNN layers,
\textit{e.g.} SchNet \citep{schutt2018schnet}, E-GNN \citep{satorras2021n}, and TFN \citep{thomas2018tensor} use distances, while DimeNet \citep{Gasteiger2020directional} uses distances and angles. 
Notably, MACE \citep{batatia2022mace} constructs a \emph{complete} basis of scalars up to arbitrary body order $k$ via Atomic Cluster Expansion \citep{dusson2019atomic}, which can be $\fG$-orbit injective if the conditions in Proposition \ref{prop:bodyorder} are met.
We can state the following about the $\text{IGWL}_{(k)}$ hierarchy and the corresponding GNNs.

\begin{restatable}{proposition}{igwlhierarchyprop}
    \label{prop:igwl-hierarchy}
    $\text{IGWL}_{(k)}$ is at least as powerful as $\text{IGWL}_{(k-1)}$. For $k \leq 5$, $\text{IGWL}_{(k)}$ is strictly more powerful than $\text{IGWL}_{(k-1)}$.
\end{restatable}

Finally, we show that $\text{IGWL}_{(2)}$ is equivalent to WL when all the pairwise distances are the same. A similar observation was recently made by \citet{pozdnyakov2022incompleteness}.

\begin{restatable}{proposition}{onewligwlprop}\label{prop:one-wl-igwl}
    Let $\gG_1 = (\mA_1, \mS_1, \vec{\mX}_1)$ and $\gG_2 = (\mA_2, \mS_2, \vec{\mX}_2)$ be two geometric graphs with the property that all edges have equal length.
    Then, $\text{IGWL}_{(2)}$ distinguishes the two graphs if and only if WL can distinguish the attributed graphs $(\mA_1, \mS_1)$ and $(\mA_1, \mS_1)$.
\end{restatable}

This equivalence points to limitations of distance-based $\fG$-invariant models like SchNet \citep{schutt2018schnet}. 
These models suffer from all well-known failure cases of WL, \textit{e.g.} they cannot distinguish two equilateral triangles from the regular hexagon \citep{Gasteiger2020directional}.


\section{Synthetic Experiments on Expressivity}
\label{sec:expt}


\begin{table*}[t!]
    \hfill
    \begin{subtable}[h]{0.28\textwidth}
        \centering
        \includegraphics[width=\linewidth]{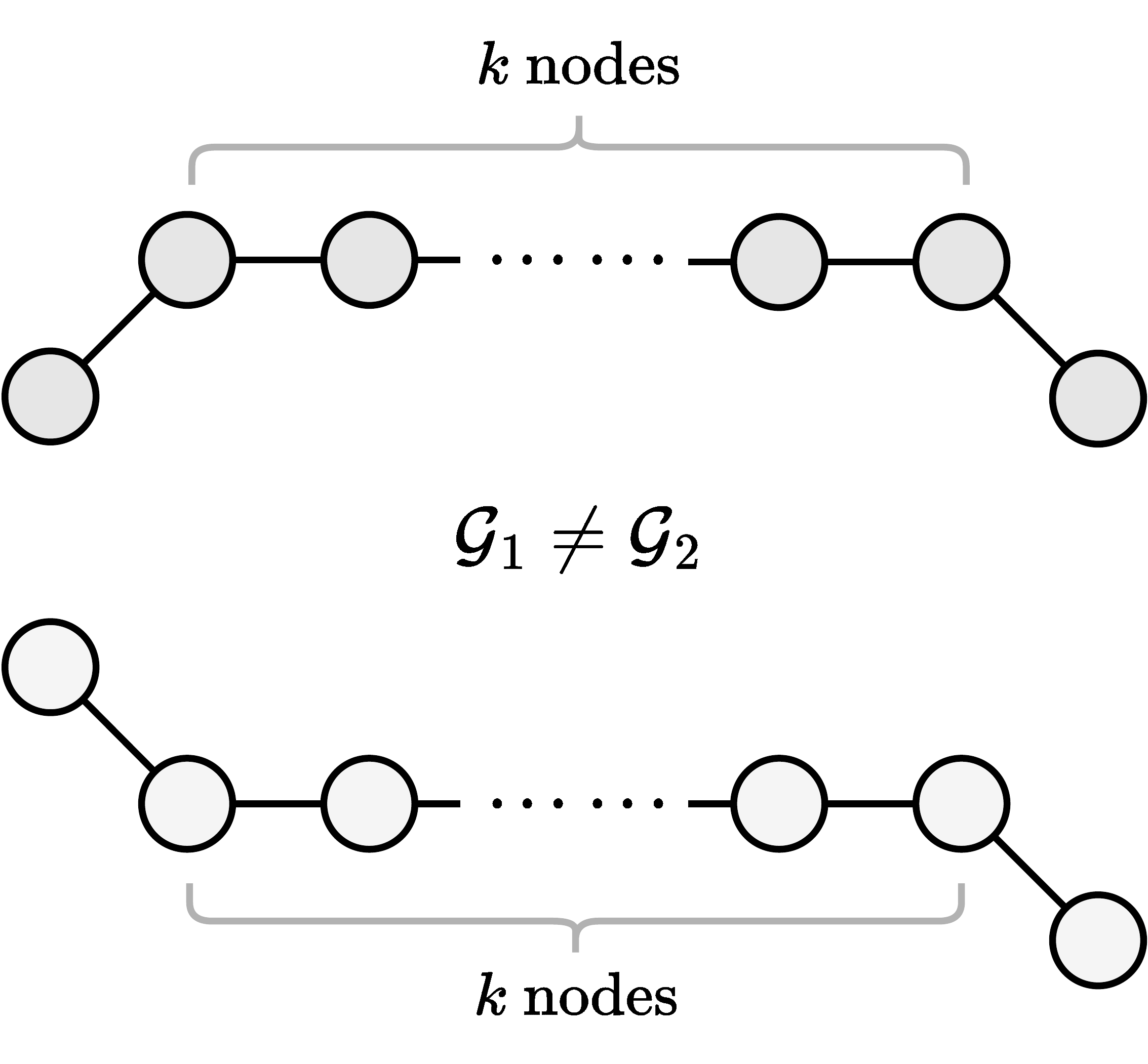}
    \end{subtable}
    \hfill
    \begin{subtable}[h]{0.7\textwidth}
        \centering
        \resizebox{\linewidth}{!}{
        \begin{tabular}{clccccc}
            \toprule
            & ($k=\mathbf{4}$-chains) & \multicolumn{5}{c}{\textbf{Number of layers}} \\
            & \textbf{GNN Layer} & $\lfloor \frac{k}{2} \rfloor$ & \cellcolor{gray!10} $\lfloor \frac{k}{2} \rfloor + 1 = \mathbf{3}$ & $\lfloor \frac{k}{2} \rfloor + 2$ & $\lfloor \frac{k}{2} \rfloor + 3$ & $\lfloor \frac{k}{2} \rfloor + 4$ \\
            \midrule
            \multirow{5}{*}{\rotatebox[origin=c]{90}{Equiv.}} &
            \gray{GWL} & \gray{50\%} & \gray{\textbf{100\%}} & \gray{\textbf{100\%}} & \gray{\textbf{100\%}} & \gray{\textbf{100\%}} \\
            & E-GNN & 50.0 {\scriptsize ± 0.0} & \cellcolor{red!10} 50.0 {\scriptsize ± 0.0} & \cellcolor{red!10} 50.0 {\scriptsize ± 0.0} & \cellcolor{red!10} 50.0 {\scriptsize ± 0.0} & \cellcolor{green!10} \textbf{100.0 {\scriptsize ± 0.0}} \\
            & GVP-GNN & 50.0 {\scriptsize ± 0.0} & \cellcolor{green!10} \textbf{100.0 {\scriptsize ± 0.0}} & \cellcolor{green!10} \textbf{100.0 {\scriptsize ± 0.0}} & \cellcolor{green!10} \textbf{100.0 {\scriptsize ± 0.0}} & \cellcolor{green!10} \textbf{100.0 {\scriptsize ± 0.0}} \\
            & TFN & 50.0 {\scriptsize ± 0.0} & \cellcolor{red!10} 50.0 {\scriptsize ± 0.0} & \cellcolor{red!10} 50.0 {\scriptsize ± 0.0} & \cellcolor{green!10} \textbf{80.0 {\scriptsize ± 24.5}} & \cellcolor{green!10} \textbf{85.0 {\scriptsize ± 22.9}} \\
            & MACE & 50.0 {\scriptsize ± 0.0} & \cellcolor{green!10} \textbf{90.0 {\scriptsize ± 20.0}} & \cellcolor{green!10} \textbf{90.0 {\scriptsize ± 20.0}} & \cellcolor{green!10} \textbf{95.0 {\scriptsize ± 15.0}} & \cellcolor{green!10} \textbf{95.0 {\scriptsize ± 15.0}} \\
            \midrule
            \multirow{4}{*}{\rotatebox[origin=c]{90}{Inv.}} &
            \gray{IGWL} & \gray{50\%} & \gray{50\%} & \gray{50\%} & \gray{50\%} & \gray{50\%} \\
            & SchNet & 50.0 {\scriptsize {\scriptsize ± 0.0}} & 50.0 {\scriptsize {\scriptsize ± 0.0}} & 50.0 {\scriptsize {\scriptsize ± 0.0}} & 50.0 {\scriptsize {\scriptsize ± 0.0}} & 50.0 {\scriptsize {\scriptsize ± 0.0}} \\
            & DimeNet & 50.0 {\scriptsize {\scriptsize ± 0.0}} & 50.0 {\scriptsize {\scriptsize ± 0.0}} & 50.0 {\scriptsize {\scriptsize ± 0.0}} & 50.0 {\scriptsize {\scriptsize ± 0.0}} & 50.0 {\scriptsize {\scriptsize ± 0.0}} \\
            & SphereNet & 50.0 {\scriptsize {\scriptsize ± 0.0}} & 50.0 {\scriptsize {\scriptsize ± 0.0}} & 50.0 {\scriptsize {\scriptsize ± 0.0}} & 50.0 {\scriptsize {\scriptsize ± 0.0}} & 50.0 {\scriptsize {\scriptsize ± 0.0}} \\
            \cmidrule{2-7}
            & SchNet$_{\text{full graph}}$ & \cellcolor{green!10} \textbf{100.0 {\scriptsize ± 0.0}} & \cellcolor{green!10} \textbf{100.0 {\scriptsize ± 0.0}} & \cellcolor{green!10} \textbf{100.0 {\scriptsize ± 0.0}} & \cellcolor{green!10} \textbf{100.0 {\scriptsize ± 0.0}} & \cellcolor{green!10} \textbf{100.0 {\scriptsize ± 0.0}} \\
            & SchNet$_{\text{global feat}}$ & \cellcolor{green!10} \textbf{100.0 {\scriptsize ± 0.0}} & \cellcolor{green!10} \textbf{100.0 {\scriptsize ± 0.0}} & \cellcolor{green!10} \textbf{100.0 {\scriptsize ± 0.0}} & \cellcolor{green!10} \textbf{100.0 {\scriptsize ± 0.0}} & \cellcolor{green!10} \textbf{100.0 {\scriptsize ± 0.0}} \\
            \bottomrule
        \end{tabular}
        }
    \end{subtable}
    \hfill
    \caption{\textit{$k$-chain geometric graphs.} 
    $k$-chains are $(\lfloor \frac{k}{2} \rfloor + 1)$-hop distinguishable and $(\lfloor \frac{k}{2} \rfloor + 1)$ GWL iterations are theoretically sufficient to distinguish them. 
    We train geometric GNNs with an increasing number of layers to distinguish $k=4$-chains.
    \textbf{$\fG$-equivariant GNNs may require more iterations that prescribed by GWL, pointing to preliminary evidence of oversquashing when geometric information is propagated across multiple layers using fixed dimensional feature spaces.
    }
    IGWL and $\fG$-invariant GNNs are unable to distinguish $k$-chains for any $k \geq 2$ and $\fG = O(3)$.
    $\fG$-invariant GNNs with precomputed non-local features (volume of bounding box) or message passing on fully connected graphs can trivially solve the task.
    Anomalous results are marked in \colorbox{red!10}{red} and expected results in \colorbox{green!10}{green}.
    }
    \label{tab:kchains}
\end{table*}


\textbf{Overview. }
We provide synthetic experiments to supplement our theoretical results and highlight practical challenges in building maximally powerful geometric GNNs, \textit{s.a.} oversquashing of geometric information with increased depth (Table \ref{tab:kchains}) and the need for higher order tensors in $\fG$-equivariant GNNs (Table \ref{tab:rotsym}).
We hope that the accompanying codebase, \texttt{geometric-gnn-dojo}\footnote{Code available on GitHub: \url{https://github.com/chaitjo/geometric-gnn-dojo}}, can be a pedagogical resource for advancing geometric GNN architectures in a principled manner. 
See Appendix \ref{app:setup} for details on experimental setup and hyperparameters.


\subsection{Experiment in Table \ref{tab:kchains}: Depth, non-local geometric properties, and oversquashing}
\label{app:setup:kchains}

GWL assumes perfect propagation of $\fG$-equivariant geometric information at each iteration, which implies that the test can be run for any number of iterations without loss of information.
In geometric GNNs, $\fG$-equivariant information is propagated via summing features from multiple layers in fixed dimensional spaces, which may lead to distortion or loss of information from distant nodes.

\textbf{Experiment. }
To study the practical implications of depth in propagating geometric information beyond local neighbourhoods, 
we consider $k$-chain geometric graphs which generalise the examples from \citet{schutt2021equivariantmp}. 
Each pair of $k$-chains consists of $k+2$ nodes with $k$ nodes arranged in a line and differentiated by the orientation of the $2$ end points.
Thus, $k$-chain graphs are $(\lfloor \frac{k}{2} \rfloor + 1)$-hop distinguishable, and $(\lfloor \frac{k}{2} \rfloor + 1)$ GWL iterations  are theoretically sufficient to distinguish them.
In Table~\ref{tab:kchains}, we train $\fG$-equivariant and $\fG$-invariant GNNs with an increasing number of layers to distinguish $k$-chains.

\textbf{Results. }
Despite the supposed simplicity of the task, we find that popular $\fG$-equivariant GNNs such as E-GNN \citep{satorras2021n} and TFN \citep{thomas2018tensor} may require more iterations that prescribed by GWL.
Notably, for chains larger than $k=4$, all $\fG$-equivariant GNNs tended to require more than $(\lfloor \frac{k}{2} \rfloor + 1)$ iterations to solve the task.
Additionally, IGWL and $\fG$-invariant GNNs are unable to distinguish $k$-chains.

Table \ref{tab:kchains} points to preliminary evidence of the \emph{oversquashing} phenomenon \citep{alon2021on, topping2022understanding} for equivariant features in geometric GNNs.
The issue is most evident for E-GNN, which uses a single vector feature to aggregate and propagate geometric information.
This may have implications in modelling macromolecules where long-range interactions often play important roles.




\begin{table*}[t!]
    \hfill
    \begin{subtable}[h]{0.25\textwidth}
        \centering
        \includegraphics[width=\linewidth]{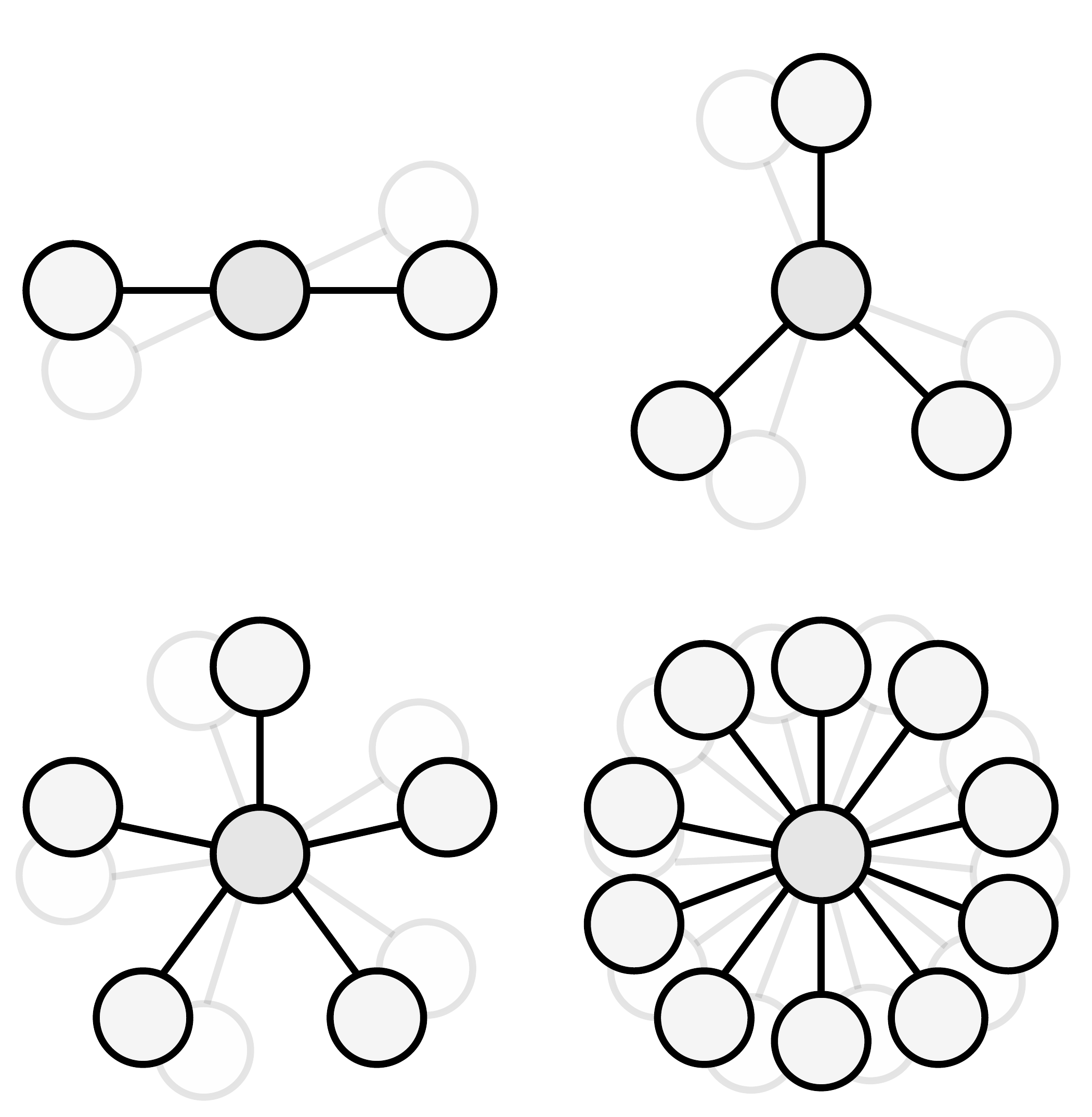}
    \end{subtable}
    \hfill
    \begin{subtable}[h]{0.65\textwidth}
        \centering
        \resizebox{\linewidth}{!}{
        \begin{tabular}{clcccc}
            \toprule
            & & \multicolumn{4}{c}{\textbf{Rotational symmetry}} \\
            & \textbf{GNN Layer} & 2 fold & 3 fold & 5 fold & 10 fold \\
            \midrule
            \multirow{2}{*}{\rotatebox[origin=c]{90}{Cart.}}
            & E-GNN$_{L=1}$ & \cellcolor{red!10} 50.0 {\scriptsize ± 0.0} & 50.0 {\scriptsize ± 0.0} & 50.0 {\scriptsize ± 0.0} & 50.0 {\scriptsize ± 0.0} \\
            & GVP-GNN$_{L=1}$ & \cellcolor{red!10} 50.0 {\scriptsize ± 0.0} & 50.0 {\scriptsize ± 0.0} & 50.0 {\scriptsize ± 0.0} & 50.0 {\scriptsize ± 0.0} \\
            \midrule
            \multirow{5}{*}{\rotatebox[origin=c]{90}{Spherical}}
            & TFN/MACE$_{L=1}$ & \cellcolor{red!10} 50.0 {\scriptsize ± 0.0} & 50.0 {\scriptsize ± 0.0} & 50.0 {\scriptsize ± 0.0} & 50.0 {\scriptsize ± 0.0} \\
            & TFN/MACE$_{L=2}$ & \cellcolor{green!10} \textbf{100.0 {\scriptsize ± 0.0}} & 50.0 {\scriptsize ± 0.0} & 50.0 {\scriptsize ± 0.0} & 50.0 {\scriptsize ± 0.0} \\
            & TFN/MACE$_{L=3}$ & \cellcolor{green!10} \textbf{100.0 {\scriptsize ± 0.0}} & \cellcolor{green!10} \textbf{100.0 {\scriptsize ± 0.0}} & 50.0 {\scriptsize ± 0.0} & 50.0 {\scriptsize ± 0.0} \\
            & TFN/MACE$_{L=5}$ & \cellcolor{green!10} \textbf{100.0 {\scriptsize ± 0.0}} & \cellcolor{green!10} \textbf{100.0 {\scriptsize ± 0.0}} & \cellcolor{green!10} \textbf{100.0 {\scriptsize ± 0.0}} & 50.0 {\scriptsize ± 0.0} \\
            & TFN/MACE$_{L=10}$ & \cellcolor{green!10} \textbf{100.0 {\scriptsize ± 0.0}} & \cellcolor{green!10} \textbf{100.0 {\scriptsize ± 0.0}} & \cellcolor{green!10} \textbf{100.0 {\scriptsize ± 0.0}} & \cellcolor{green!10} \textbf{100.0 {\scriptsize ± 0.0}} \\
            \bottomrule
        \end{tabular}
        }
    \end{subtable}
    \hfill
    \caption{\textit{Rotationally symmetric structures.}
    We train single layer $\fG$-equivariant GNNs to distinguish two \emph{distinct} rotated versions of each $L$-fold symmetric structure.
    \textbf{
    We find that layers using order $L$ tensors are unable to identify the orientation of structures with rotation symmetry higher than $L$-fold.
    }
    This issue is particularly prevalent for layers using cartesian vectors (tensor order 1).
    }
    \label{tab:rotsym}
\end{table*}


\subsection{Experiment in Table \ref{tab:rotsym}: Higher order tensors and rotationally symmetric structures}
\label{app:setup:rotsym}

In addition to perfect propagation, GWL is also able to injectively aggregate $\fG$-equivariant information by making use of an auxiliary nested geometric object $\vg_i$.
On the other hand, $\fG$-equivariant GNNs aggregate geometric information via summing neighbourhood features represented by Cartesian vectors (tensor order 1) or higher order spherical tensors.
This choice of basis often comes with trade-offs between computational tractability and empirical performance.

\textbf{Experiment. }
To demonstrate the utility of higher order tensors in $\fG$-equivariant GNNs, we study how rotational symmetries interact with tensor order. 
In Table \ref{tab:rotsym}, we evaluate current $\fG$-equivariant layers on their ability to distinguish the orientation of structures with rotational symmetry. 
An $L$-fold symmetric structure does not change when rotated by an angle $\frac{2\pi}{L}$ around a point (in 2D) or axis (3D).
We consider two \emph{distinct} rotated versions of each symmetric structure and train single layer $\fG$-equivariant GNNs to classify the two orientations using the updated equivariant features.

\textbf{Results. }
We find that layers using order $L$ spherical tensors are unable to identify the orientation of structures with rotation symmetry higher than $L$-fold, \textit{i.e.} two distinct rotated versions of the input having the same equivariant features.
We attribute this observation to spherical harmonics, which serve as an orthonormal basis for spherical tensor features and exhibit rotational symmetry themselves.

Similar to the Fourier expansion for signals, the spherical harmonic expansion is employed for converting Cartesian vectors to spherical signals in $\fG$-equivariant GNNs. 
The tensor order of the spherical harmonic bases determines the rate of oscillation of the approximated function on the sphere. 
We believe that this oscillation rate is closely linked to the rotational fold of a set of symmetric vectors.

In the Fourier expansion, it is not feasible to accurately approximate a high-frequency function solely using low-frequency sinusoidal waves. Similarly, when truncating the spherical harmonic expansion to an order lower than the fold of the rotational symmetry, the rotationally symmetric vectors act as a higher frequency function. Consequently, the lower frequency bases cannot preserve the orientation of these vectors.




Layers such as E-GNN \citep{satorras2021n} and GVP-GNN \citep{jing2020learning} using Cartesian vectors are popular as higher order tensors can be computationally intractable for many applications. 
However, E-GNN and GVP-GNN are particularly poor at discriminating the orientation of rotationally symmetric structures. 
This may have implications for the modelling of periodic materials which naturally exhibit such symmetries \citep{levine1984quasicrystals}.

\section{Discrimination and Universality}
\label{sec:universality}

\textbf{Overview. } 
Following \citet{chen2019equivalence}, we study the equivalence between the universal approximation capabilities of geometric GNNs~\citep{dym2020universality, villar2021scalars} and the perspective of discriminating geometric graphs introduced by GWL, generalising their results to any isomorphism induced by a compact Lie group $\fG$. 
We further study the number of invariant aggregators that are required in a continuous setting to distinguish any two neighbourhoods. 
Proofs are available in Appendix \ref{app:universality}.

In the interest of generality, we use a general space $X$ acted upon by a compact group $\fG$ and we are interested in the capacity of $\fG$-invariant functions over $X$ to separate points in $Y$. 
The restriction to a smaller subset $Y$ is useful because we would like to separately consider the case when $Y$ is countable due to the use of countable features. Therefore, in general, the action of $\fG$ on $Y$ might not be strictly defined since it might yield elements outside $Y$.
For our setting, the reader could take $X = (\sR^d \times \sR^f)^{n \times n}$ to be the space of geometric graphs and $Y = \gX^{n \times n}$, where $\gX \subseteq \sR^d \times \sR^f$.

\begin{definition}
    \label{def:discriminating}
    Let $\fG$ be a compact group and $\gC$ a collection of $\fG$-invariant functions from a set $X$ to $\sR$. For a subset $Y \subseteq X$, we say the $\gC$ is \textbf{pairwise} $\mathbf{Y_\fG}$\textbf{ discriminating} if for any $y_1, y_2 \in Y$ such that $y_1 \not\simeq y_2$, there exists a function $h \in \gC$ such that $h(y_1) \neq h(y_2)$.
\end{definition}

We note here that $h$ is not necessarily injective, \textit{i.e.} there might be $y_1', y_2'$ for which $h(y_1') = h(y_1')$. Therefore, pairwise discrimination is a weaker notion of discrimination than the one GWL uses.

\begin{definition}
    \label{def:universally_approximating}
    Let $\fG$ be a compact group and $\gC$ a collection of $\fG$-invariant functions from $X$ to $\sR$. For $Y \subseteq X$, we say the $\gC$ is \textbf{universally approximating} over $Y$ if for all  $\fG$-invariant functions $f$ from $X$ to $\sR$ and for all $\eps > 0$, there exists $h_{\varepsilon, f} \in \gC$ such that $\norm{f - h_{\varepsilon, f}}_Y := \sup_{y \in Y} | f(y) - h(y)| < \varepsilon$.
\end{definition}

\textbf{Countable Features. } We first focus on the countable feature setting that the GWL test operates in. Therefore, we will assume that $Y$ is a countable subset of $X$.

\begin{restatable}{theorem}{FiniteUnivImpliesDiscrim}\label{theo:finite_univ_implies_discrim}
    If $\gC$ is universally approximating over $Y$, then $\gC$ is also pairwise $Y_\fG$ discriminating.
\end{restatable}

This result further motivates the interest in discriminating geometric graphs, since a model that cannot distinguish two non-isomorphic geometric graphs is not universal. By further assuming that $Y$ is finite, we obtain a result in the opposite direction. Given a collection of functions $\gC$, we define like in \citet{chen2019equivalence} the class $\gC^{+L}$ given by all the functions of the form $\textrm{MLP}([f_1(x),\ldots,f_k(x)])$ with $f_i \in \gC$ and finite $k$, where the MLP has $L$ layers with ReLU hidden activations.

\begin{restatable}{theorem}{FiniteDiscrimImpliesUniv}\label{theo:finite_discrim_implies_univ}
    If $\gC$ is pairwise $Y_\fG$ discriminating, then $\gC^{+2}$ is universally approximating over $Y$.
\end{restatable}

\textbf{Continuous Features. } The symmetries characterising geometric graphs are naturally continuous (\textit{e.g.} rotations). Therefore, it is natural to ask how the results above translate to \emph{continuous} $\fG$-invariant functions over a continuous subspace $Y$. 
Therefore, for the rest of this section, we assume that $(X, d)$ is a metric space, $Y$ is a compact subset of $X$ and $\fG$ acts continuously on $X$.

\begin{restatable}{theorem}{UniversalCanSeparate}\label{theo:universal_can_separate}
    If $\gC$ can universally approximate any continuous $\fG$-invariant functions on $Y$, then $\gC$ is also pairwise $Y_\fG$ discriminating.
\end{restatable}

With the additional assumption that $X = Y$, we can also prove the converse.

\begin{restatable}{theorem}{SeparationImpliesUnivers}\label{theo:separation_implies_univers}
    If $\gC$, a class of functions over $Y$, is pairwise $Y_\fG$ discriminating, then $\gC^{+2}$ can also universally approximate any continuous function over $Y$.
\end{restatable}

We now produce an estimate for the number of aggregators needed to learn \emph{continuous} orbit-space injective functions on a manifold $X$ based on results from differential geometry \citep{lee2013smooth}.
A group $\fG$ acts freely on $X$ if $\fg x = x$ implies $\fg = e_\fG$, where $e_\fG$ is the identity element in $\fG$.  

\begin{restatable}{theorem}{NumAgg}\label{theo:num_agg}
    Let $X$ be a smooth $n$-dim manifold and $\fG$ an $m$-dim compact Lie group acting continuously on $X$. Suppose there exists a smooth submanifold $Y$ of $X$ of the same dimension such that $\fG$ acts freely on it. Then, any $\fG$-orbit injective function $f: X \to \sR^d$ requires that $d \geq n-m$.
\end{restatable}

We now apply this theorem to the local aggregation operation performed by geometric GNNs. Let $X = \sR^{n \times d}$ and $\fG = S_n \times O(d)$ or $S_n \times SO(d)$. Let $\mP_\fg$ and $\mQ_\fg$ be the permutation matrix and the orthogonal matrix associated with the group element $\fg \in \fG$. Then $\fg$ acts on matrices $\mX \in X$ continuously via $\mP_\fg \mX \mQ_\fg^\top$. Then, $\fG$ orbit-space injective functions on $X$ are functions on point clouds of size $n$ that can distinguish any two different point clouds.

\begin{restatable}{theorem}{SOdAgg}\label{theo:so_d_agg}
    For $n \geq d-1 > 0$ or $n=d=1$, any continuous $S_n \times SO(d)$ orbit-space injective function $f: \sR^{n \times d} \to \sR^{q}$ requires that $q \geq nd - d(d-1)/2$.
\end{restatable}

We can also generalise this to $O(d)$, with the slightly stronger assumption that $n \geq d$.

\begin{restatable}{theorem}{OdAgg}\label{theo:o_d_agg}
    For $n \geq d > 0$, any continuous $S_n \times O(d)$ orbit-space injective function $f: \sR^{n \times d} \to \sR^{q}$ requires that $q \geq nd - d(d-1)/2$.
\end{restatable}

Overall, these results show that when working with point clouds in $\sR^3$ as is common in molecular or physical applications, at least $q = 3(n-1)$ aggregators are required. This result argues why a bigger representational width can help distinguish neighbourhoods. Finally, in the particular case of the zero-dimensional subgroup $S_n \times \{e_{\text{SO(d)}}\} \simeq S_n$ we obtain a statement holding for all $n$ and generalising a result from \citet{corso2020principal} regarding the aggregators for non-geometric GNNs. The original PNA result considers the case $d=1$ and here we extend it to arbitrary $d$.

\begin{restatable}{proposition}{generalisedPNA}\label{prop:pna}
    Any $S_n$-invariant injective function $f: $ $\sR^{n \times d} \to \sR^q$ requires $q \geq nd$.
\end{restatable}


\section{Conclusion}

We propose the Geometric Weisfeiler-Leman (GWL) framework to characterise the expressive power of geometric GNNs.
Our work fills a key research gap as standard GNNs and the associated theoretical tools are inapplicable for geometric graphs.
(1) We demonstrate the utility of GWL for understanding how key design choices influence geometric GNN expressivity, (2) contribute synthetic experiments highlighting practical challenges for building expressive models, and (3) connect GWL's discrimination-based perspective to universal approximation. 


\textbf{Future Work. }
GWL provides an abstraction to study the theoretical limits of geometric GNNs. 
In practice, it is challenging to build provably powerful models that satisfy the conditions of Proposition \ref{theo:conditions}, as GWL relies on perfect neighbourhood aggregation and colouring functions.
Based on the understanding gained from GWL, future work will develop maximally powerful, \emph{practical} geometric GNNs.

\textbf{Acknowledgements. }
We would like to thank Andrew Blake, Challenger Mishra, Charles Harris, Dávid Kovács, Erik Thiede, Gabor Csanyi, Gregor Simm, Hannes Stärk, Ilyes Batatia, Iulia Duta, Justin Tan, Limei Wang, Mario Geiger, Petar Veličković, Ramon Vinãs, Rob Hesselink, Soledad Villar, Weihua Hu, and the anonymous reviewers for helpful discussions.
CKJ was supported by the A*STAR Singapore National Science Scholarship (PhD).
SVM was supported by the UKRI Centre for Doctoral Training in Application of Artificial Intelligence to the study of Environmental Risks (EP/S022961/1).

\newpage

\bibliography{references}
\bibliographystyle{icml2023}

\appendix
\onecolumn


\begin{figure*}[t!]
    \centering
    \includegraphics[width=0.75\linewidth]{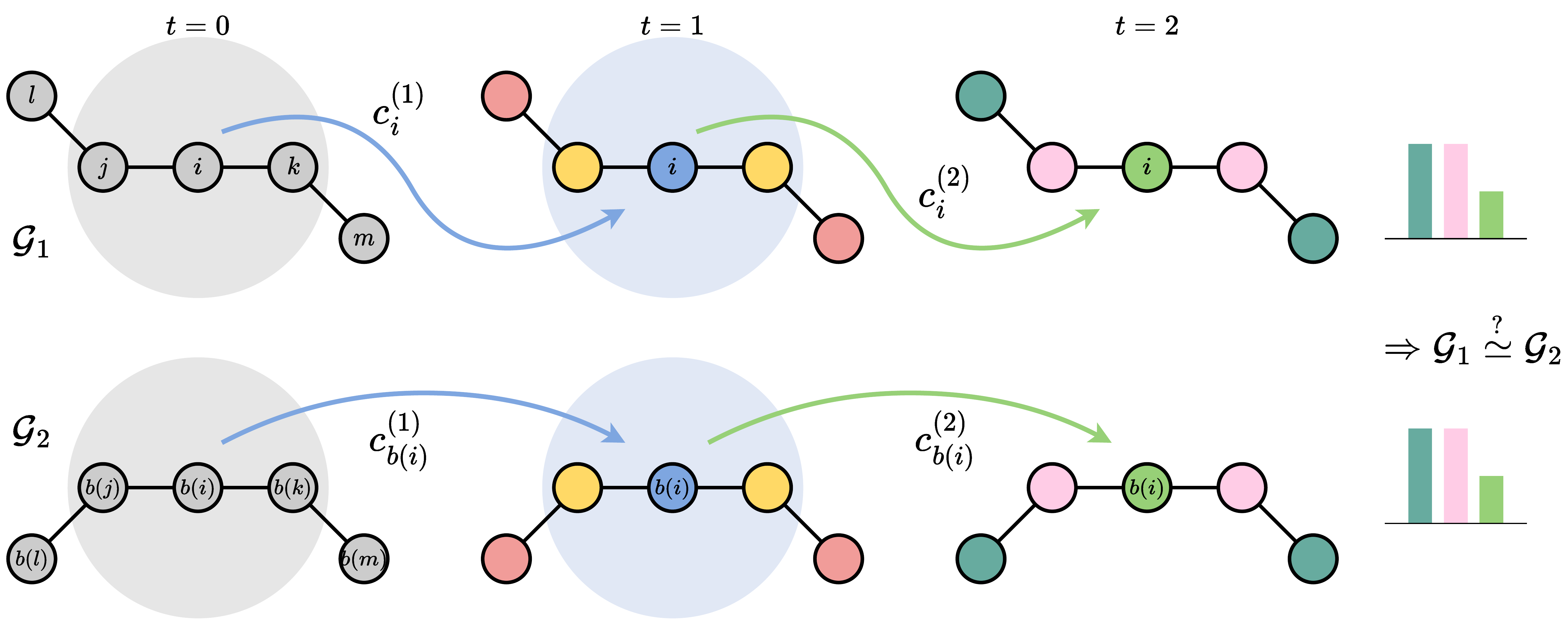}
    \caption{\textbf{Invariant GWL Test.}
        IGWL cannot distinguish $\mathcal{G}_1$ and $\mathcal{G}_2$ as they are $1$-hop identical: The $\fG$-orbit of the $1$-hop neighbourhood around each node is the same, and IGWL cannot propagate geometric orientation information beyond $1$-hop. 
    }
    \label{fig:igwl}
\end{figure*}




\section{Discussion}

\textbf{Related Work. }
In the molecular simulations community,
literature on the \textit{completeness} of atom-centred interatomic potentials has focused on distinguishing $1$-hop local neighbourhoods (point clouds) around atoms by building spanning sets for continuous, $\fG$-equivariant multiset functions \citep{shapeev2016moment, drautz2019atomic, dusson2019atomic, pozdnyakov2020incompleteness}.
GWL generalises and extends this analysis to generic geometric graph isomorphism problems beyond local atom-centred neighbourhoods.

Recent theoretical work on geometric GNNs and their universality \citep{dym2020universality, villar2021scalars, Gasteiger2021gemnet, jing2020learning} has shown that architectures such as TFN, GemNet and GVP-GNN can be universal approximators of continuous, $\fG$-equivariant or $\fG$-invariant multiset functions over point clouds, \textit{i.e.} fully connected graphs.
In contrast, the GWL framework studies the expressive power of geometric GNNs operating on \emph{sparse} graphs from the perspective of discriminating geometric graphs.
The discrimination lens is potentially more granular and practically insightful than universality.
A model may either be universal or not. 
On the other hand, there could be multiple degrees of discrimination depending on the classes of geometric graphs that can and cannot be distinguished, which our work aims to formalise.


\textbf{Limitations. }
\emph{Can GWL distinguish all geometric graphs in theory?}
Proposition \ref{prop:noniso} shows that GWL can distinguish any non-isomorphic geometric graphs for the practical case where the graph structure is computed using radial cutoffs.
However, for the general case where the topology of the geometric graph is independent of the coordinates, GWL may not be theoretically complete as there may exist pathological \emph{edge cases} where the test fails. 
For example, when all coordinates and vector features are set equal to zero, GWL coincides with the standard 1-WL. 
In this edge case, GWL has the same expressive power as 1-WL and inherits all well-known failure cases of 1-WL.

\emph{Does GWL characterise all classes of geometric GNNs?}
GWL is a node-centric framework for analysing the expressive power of geometric GNNs which perform local message passing \emph{without} canonical reference frames.
Non-local architectures such as GemNet-Q \citep{Gasteiger2021gemnet}, which aggregates dihedral angles from two-hop neighbourhoods, are not directly covered by GWL.
Similarly, the analysis of geometric GNNs based on canonical reference frames \citep{du2022se, wang2022comenet} is left to future work.
This class of architectures use a local or global frame of reference to scalarise equivariant features into invariant ones, offering an alternative modelling technique when canonical reference frames can easily be defined (\emph{e.g.} protein backbone structures \citep{ jumper2021highly, baek2021accurate, wang2023learning}).

Notably, \citet{wang2022comenet} introduced another notion of \emph{completeness}, different from completeness of atomic structured representations \citep{pozdnyakov2020incompleteness} which GWL builds upon.
To the best of our knowledge, building \emph{provably} complete geometric GNNs under either definition for all possible geometric graphs remains an open question that our framework can help the community explore.
For instance, \citet{wang2022comenet} assume that 4-body aggregators such as SphereNet are locally complete. 
While this is likely the case for many practically occurring geometric graphs, the results in Table \ref{tab:ihash} for the counterexamples from \citet{pozdnyakov2020incompleteness} show that SphereNet is indeed not $\fG$-orbit injective.


\newpage

\section{Additional Background on Geometric Graph Neural Networks}
\label{app:gnn-bg}

GNNs specialised for geometric graphs can be categorised according to the type of intermediate representations:
(1) \textit{Equivariant models}, where the intermediate features and propagated messages are geometric quantities; and
(2) \textit{Invariant models}, which only propagate local invariant scalar features such as distances and angles. 
As summarised in Figure \ref{fig:axes}, the GWL framework can be used to characterise the expressive power of the following classes of geometric GNNs\footnote{See \citet{duval2023hitchhiker} for a comprehensive introduction to geometric GNN architectures.}.


\begin{figure}[t!]
    \centering
    \begin{subfigure}[b]{0.25\linewidth}
        \centering
        \includegraphics[width=\linewidth]{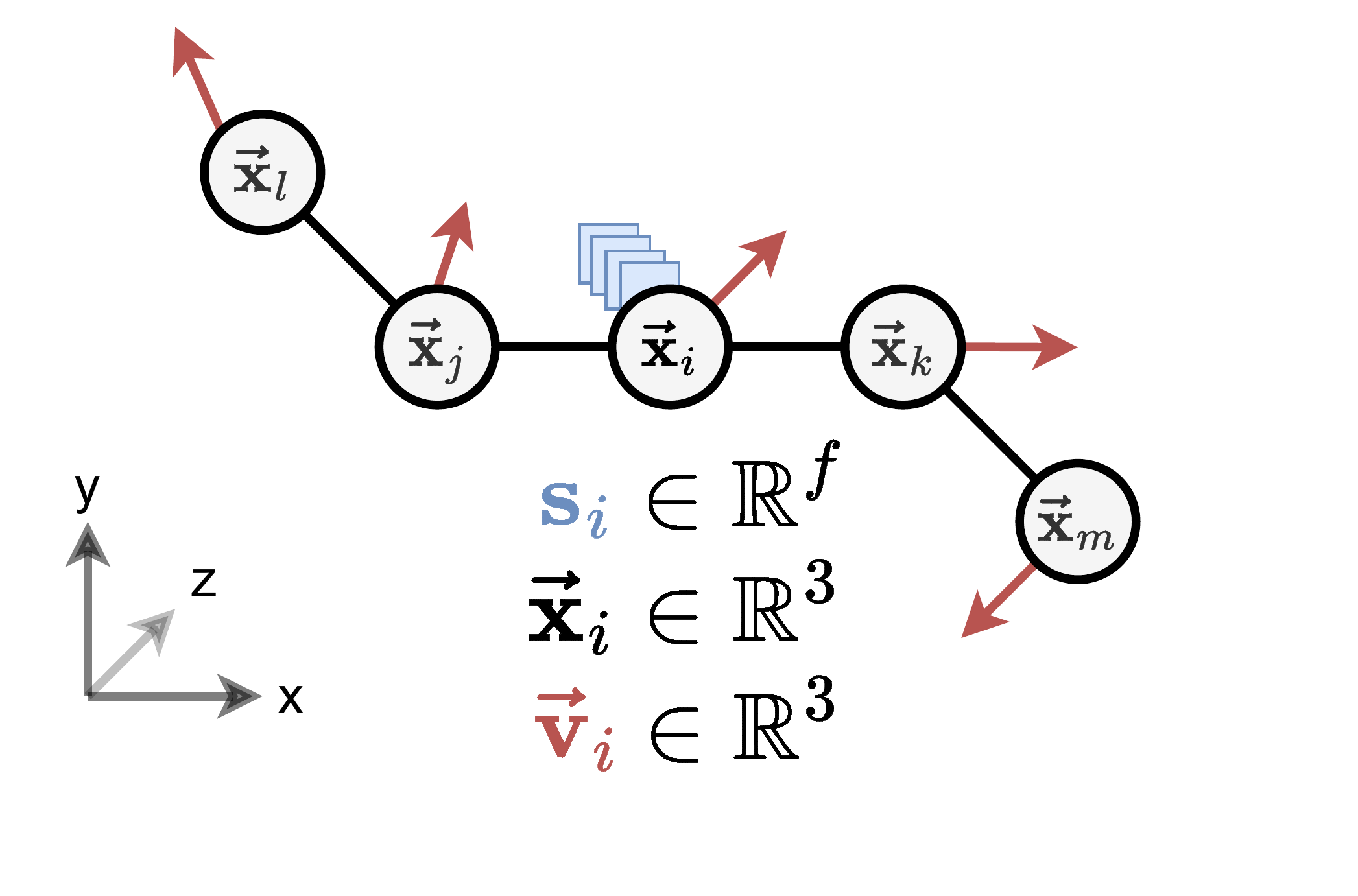}
        \caption{A geometric graph}
        \label{fig:geom-graph}
    \end{subfigure}
    \hfill
    \begin{subfigure}[b]{0.37\linewidth}
        \centering
        \includegraphics[width=\linewidth]{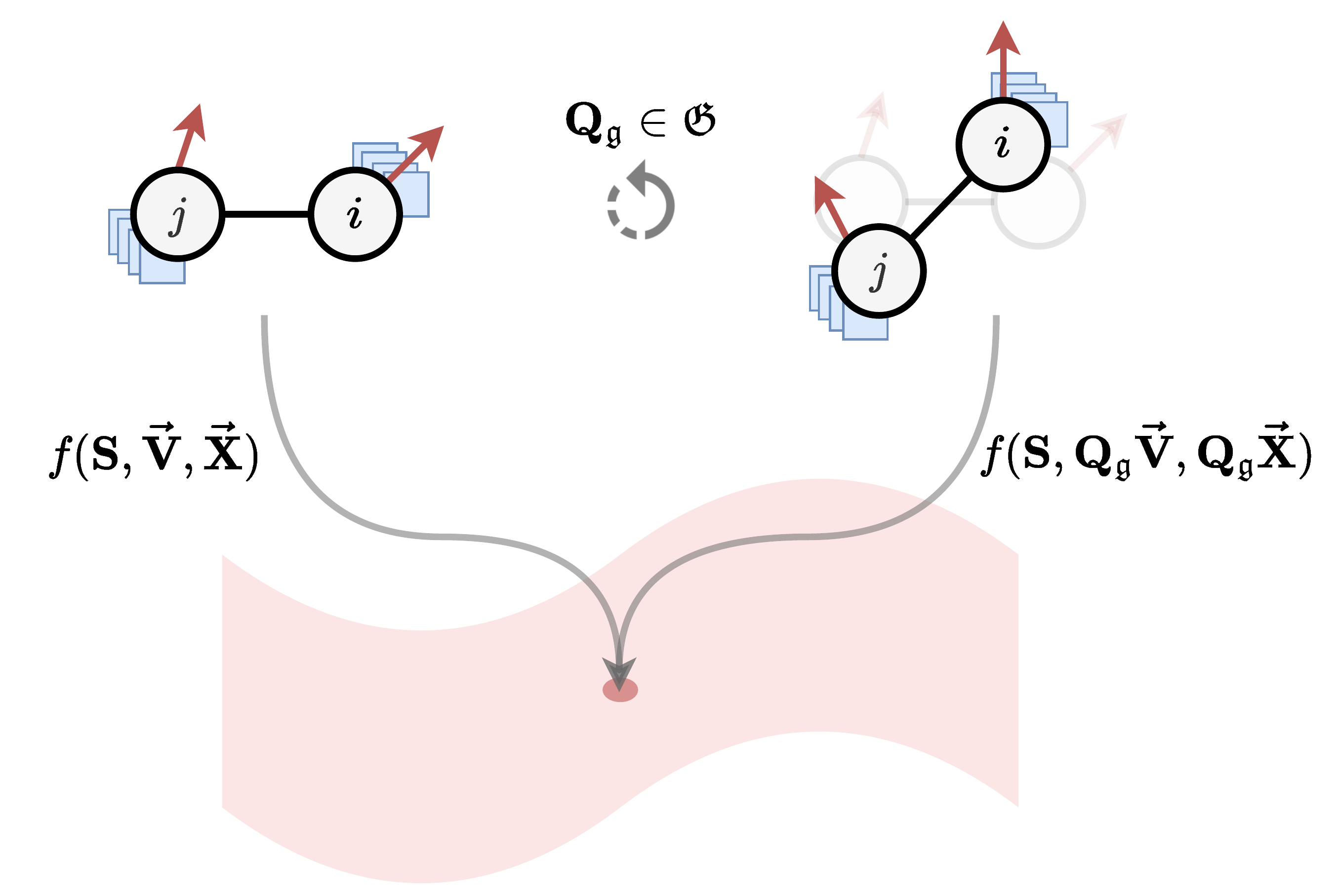}
        \caption{$\fG$-invariant function}
        \label{fig:invariant-fn}
    \end{subfigure}
    \hfill
    \begin{subfigure}[b]{0.37\linewidth}
        \centering
        \includegraphics[width=\linewidth]{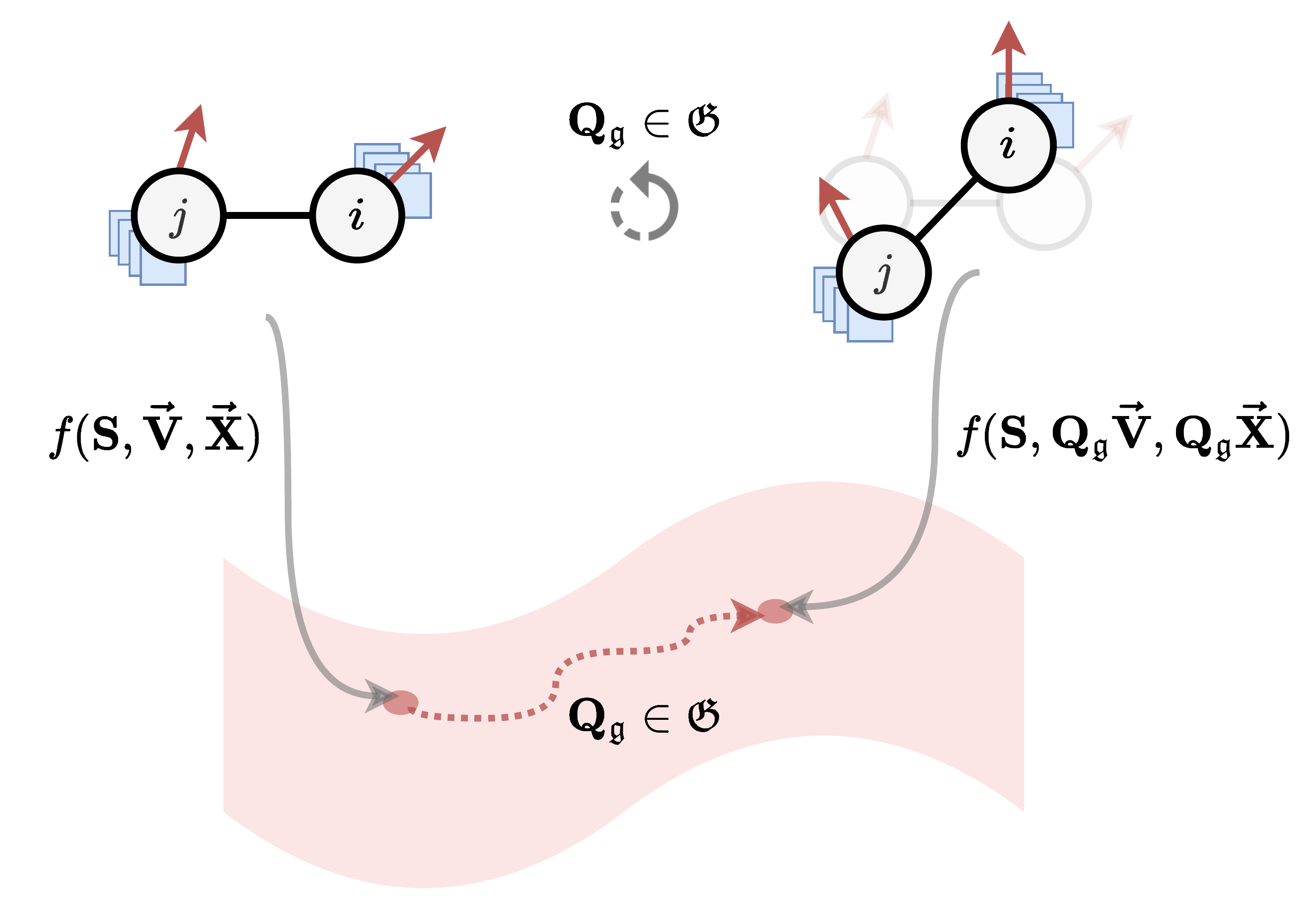}
        \caption{$\fG$-equivariant function}
        \label{fig:equivariant-fn}
    \end{subfigure}
        \caption{\textbf{Invariant and equivariant functions on geometric graphs. }
        (a) Geometric graphs are attributed graphs embedded in Euclidean space.
        The geometric attributes transform along with Euclidean transformations of the system, such as global rotations $\fG$.
        (b) The output of $\fG$-invariant functions remains unchanged under transformations of the input.
        (c) For $\fG$-equivariant functions, transformations of the input must result in the output transforming equivalently.
        }
       \label{fig:invariant-equivariant}
\end{figure}


\subsection{$\fG$-invariant GNNs}

$\fG$-invariant GNN layers aggregate scalar quantities from local neighbourhoods via scalarising the geometric information. 
Scalar features are update from iteration $t$ to $t+1$ via learnable aggregate and update functions, $\textsc{Agg}$ and $\textsc{Upd}$, respectively:
\begin{align}
    \label{eq:gnn-inv-app}
    \vs_i^{(t+1)} & \defeq \textsc{Upd} \left( \vs_i^{(t)} \ , \ \textsc{Agg} \left( \ldblbrace (\vs_i^{(t)}, \vec{\vv}_i^{(t)}), (\vs_j^{(t)},  \vec{\vv}_j^{(t)}), \vec{\vx}_{ij} \mid j \in \mathcal{N}_i \rdblbrace \right) \right).
\end{align}
\textbf{SchNet} \citep{schutt2018schnet} and \textbf{CGCNN} \citep{xie2018cgcnn} use relative distances $\Vert \vec{\vx}_{ij} \Vert$, a 2-body invariant, to scalarise local geometric information:
\begin{align}
    \vs_i^{(t+1)} & \defeq \vs_i^{(t)} + \sum_{j \in \mathcal{N}_i} f_1 \left( \vs_j^{(t)} , \ \Vert \vec{\vx}_{ij} \Vert \right) \label{eq:schnet}
\end{align}

\textbf{DimeNet} \citep{Gasteiger2020directional} and \textbf{GemNet-T} \citep{Gasteiger2021gemnet} use both distances and angles $\vec{\vx}_{ij} \cdot \vec{\vx}_{ik} $ among triplets (3-body invariants), as follows:
\begin{align}
    \vs_i^{(t+1)} & \defeq \sum_{j \in \mathcal{N}_i} f_1 \Big( \vs_i^{(t)} , \ \vs_j^{(t)} , \sum_{k \in \mathcal{N}_i \backslash \{j\}} f_2 \left( \vs_j^{(t)} , \ \vs_k^{(t)} , \ \Vert \vec{\vx}_{ij} \Vert , \ \vec{\vx}_{ij} \cdot \vec{\vx}_{ik} \right) \Big) \label{eq:dimenet}
\end{align}
The updated scalar features are both $\fG$-invariant and $T(d)$-invariant as the only geometric information used are the relative distances and angles, both of which remain unchanged under the action of $\fG$ or translations.


\begin{figure}[h!]
    \centering
    \begin{subfigure}[b]{0.22\linewidth}
        \centering
        \includegraphics[width=\linewidth]{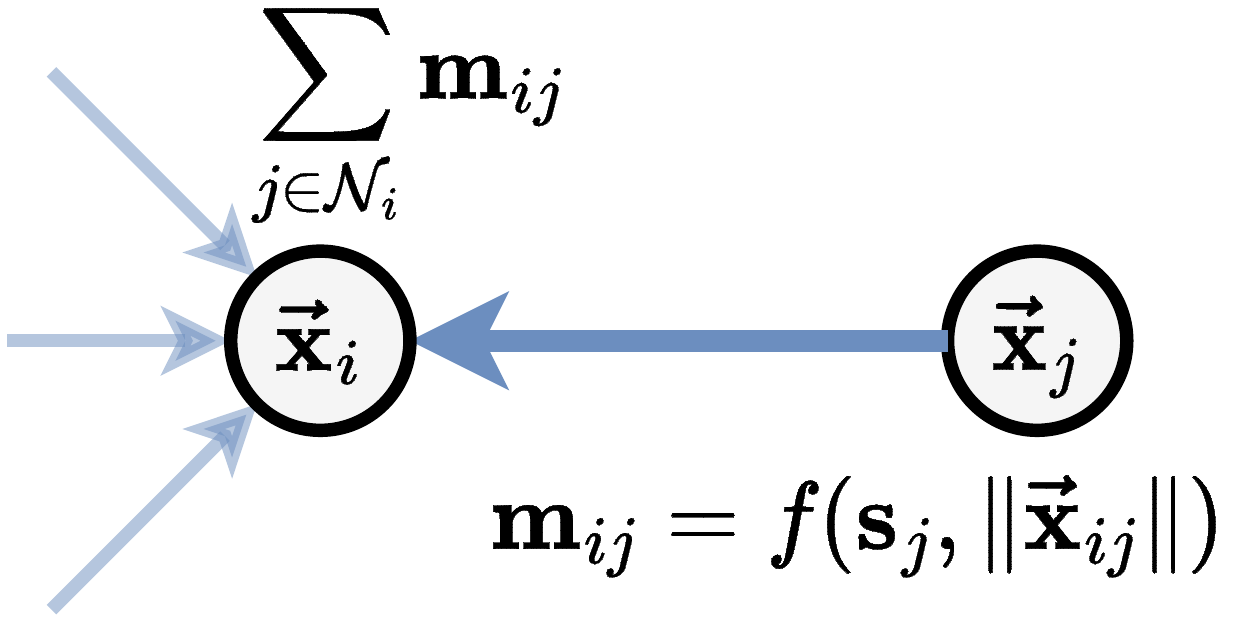}
        \vspace{6.3pt}
        \caption{SchNet}
        \label{fig:schnet}
    \end{subfigure}
    \hfill
    \begin{subfigure}[b]{0.25\linewidth}
        \centering
        \includegraphics[width=\linewidth]{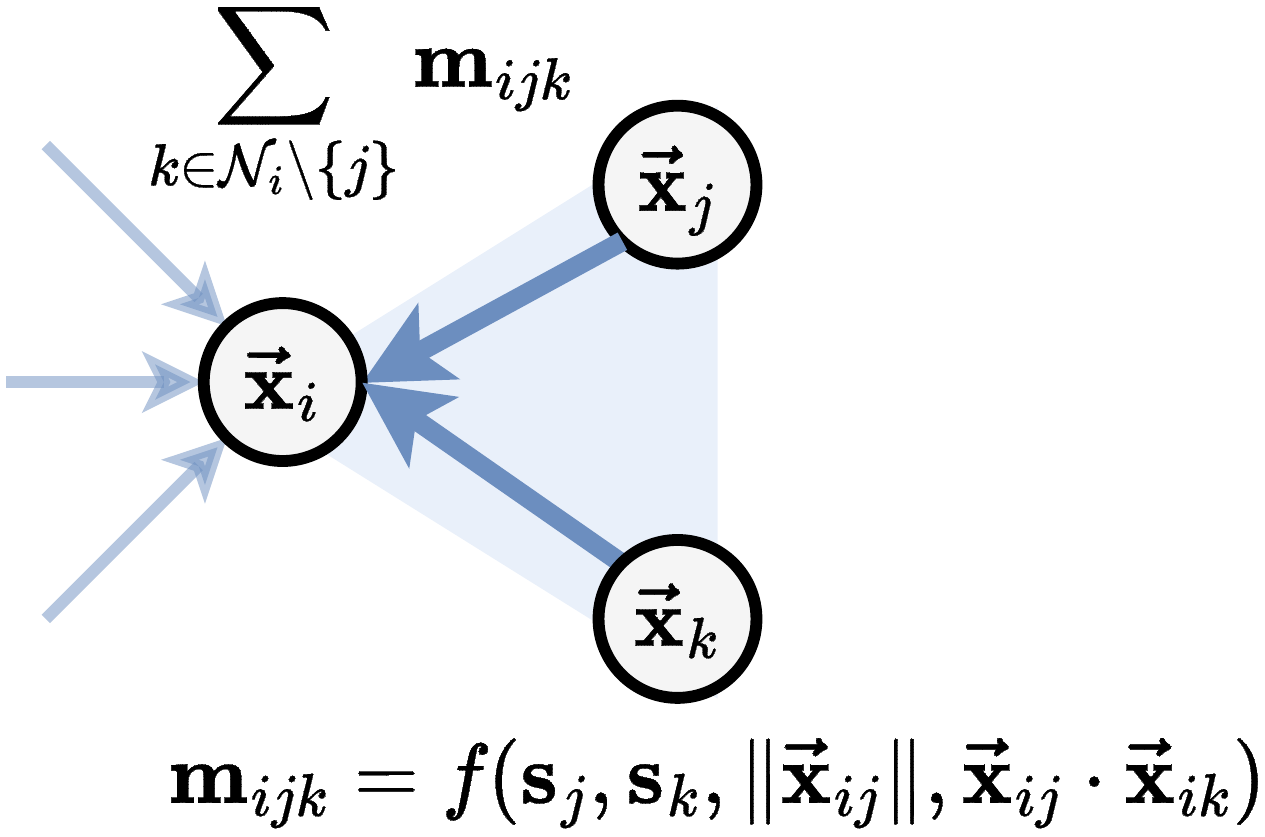}
        \caption{DimeNet}
        \label{fig:dimenet}
    \end{subfigure}
    \hfill
    \begin{subfigure}[b]{0.25\linewidth}
        \centering
        \includegraphics[width=\linewidth]{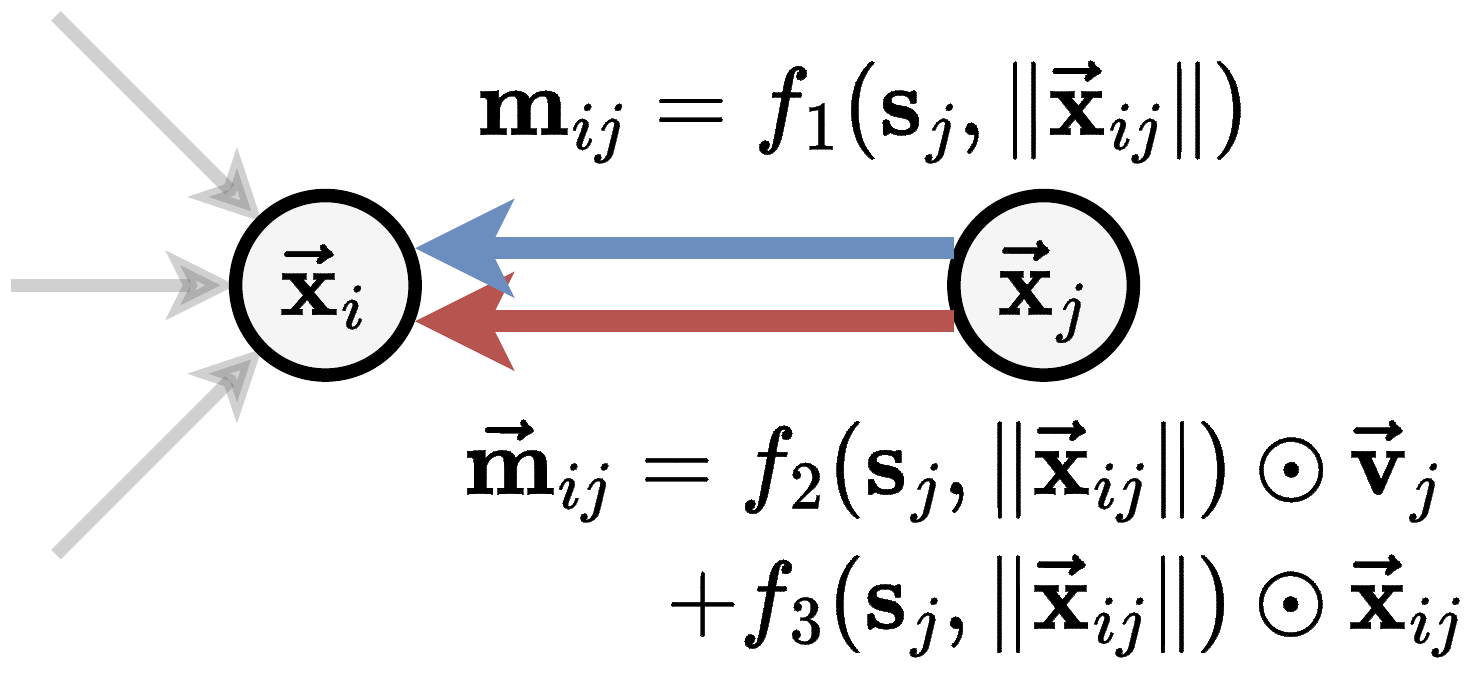}
        \vspace{0.9pt}
        \caption{PaiNN}
        \label{fig:painn}
    \end{subfigure}
    \hfill
    \begin{subfigure}[b]{0.22\linewidth}
        \centering
        \includegraphics[width=\linewidth]{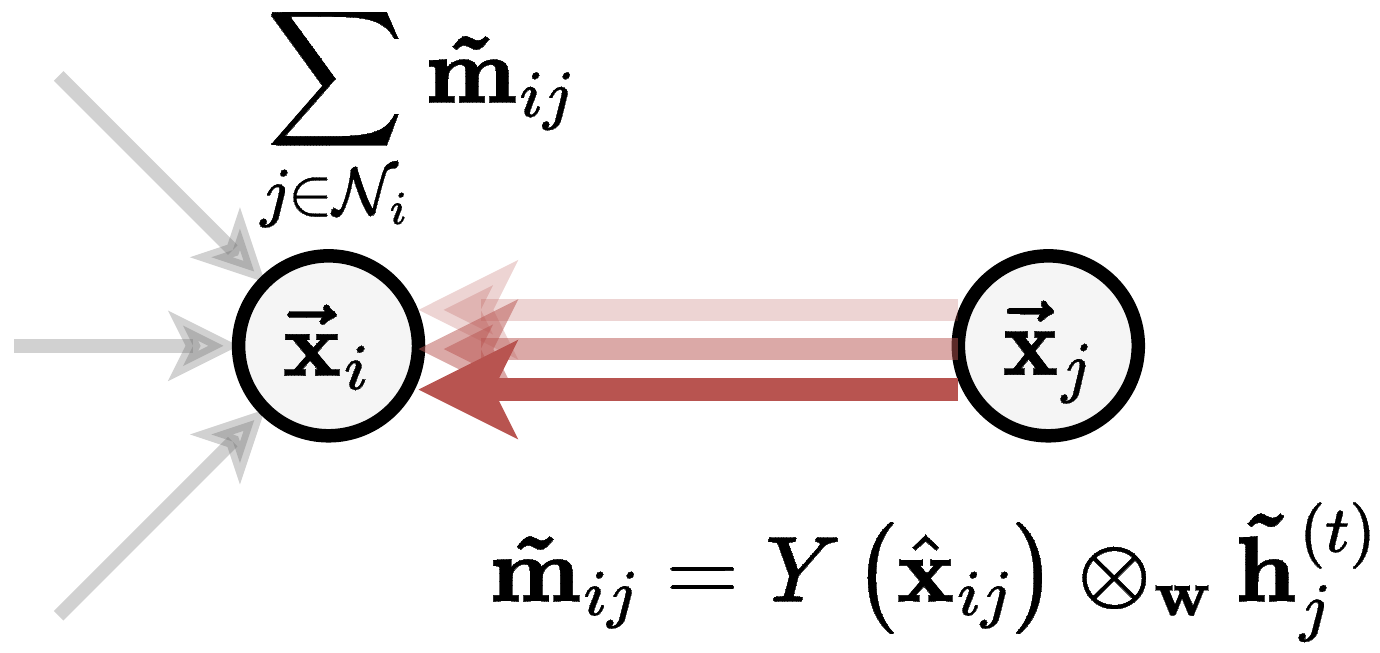}
        \vspace{7pt}
        \caption{Tensor Field Network}
        \label{fig:tfn}
    \end{subfigure}
        \caption{\textbf{Geometric GNN message passing. }
        $\fG$-invariant layers only propagate local scalar quantities such as distances (SchNet, \eqref{eq:schnet}) or distances and angles (DimeNet, \eqref{eq:dimenet}). 
        In contrast, $\fG$-equivariant layers propagated geometric quantities such as vectors and relative positions (PaiNN, \eqref{eq:painn-v}) or higher order tensors (Tensor Field Network, \eqref{eq:e3nn-1}).
        }
       \label{fig:schnet-dimenet-painn}
\end{figure}


\subsection{$\fG$-equivariant GNNs using Cartesian vectors}

$\fG$-equivariant GNN layers update both scalar and vector features by propagating scalar as well as vector messages, $\vm_i^{(t)}$ and $\vec{\vm}_i^{(t)}$, respectively:
\begin{align}
    \label{eq:gnn-equiv-app}
    \vm_i^{(t)}, \vec{\vm}_i^{(t)}     & \defeq \textsc{Agg} \left( \ldblbrace (\vs_i^{(t)}, \vec{\vv}_i^{(t)}), (\vs_j^{(t)},  \vec{\vv}_j^{(t)}), \vec{\vx}_{ij} \mid j \in \mathcal{N}_i \rdblbrace \right) & \text{(Aggregate)} \\
    \vs_i^{(t+1)}, \vec{\vv}_i^{(t+1)} & \defeq \textsc{Upd} \left( (\vs_i^{(t)}, \vec{\vv}_i^{(t)}) \ , \ (\vm_i^{(t)}, \vec{\vm}_i^{(t)}) \right)                                                         & \text{(Update)}
\end{align}
\textbf{PaiNN} \citep{schutt2021equivariantmp} interaction layers aggregate scalar and vector messages via learnt filters conditioned on the relative distance (2-body invariants):
\begin{align}
\label{eq:painn-s}
    \vm_i^{(t)} & \defeq \vs_i^{(t)} + \sum_{j \in \mathcal{N}_i} f_1 \left( \vs_j^{(t)} , \ \Vert \vec{\vx}_{ij} \Vert \right)                                                                                                                                               \\
\label{eq:painn-v}
    \vec{\vm}_i^{(t)} & \defeq \vec{\vv}_i^{(t)} + \sum_{j \in \mathcal{N}_i} f_2 \left( \vs_j^{(t)} , \ \Vert \vec{\vx}_{ij} \Vert \right) \odot \vec{\vv}_j^{(t)} + \sum_{j \in \mathcal{N}_i} f_3 \left( \vs_j^{(t)} , \ \Vert \vec{\vx}_{ij} \Vert \right) \odot \vec{\vx}_{ij}
\end{align}
\textbf{E-GNN} \citep{satorras2021n} and \textbf{GVP-GNN} \citep{jing2020learning} use similar operations.
The update step applies a gated non-linearity \citep{weiler20183d} on the vector features, which learns to scale their magnitude using their norm concatenated with the scalar features:
\begin{align}
    \label{eq:painn-u}
    \vs_i^{(t+1)} \defeq \vm_i^{(t)} + f_4 \left( \vm_i^{(t)}, \Vert \vec{\vm}_i^{(t)} \Vert \right) , \quad\quad
    \vec{\vv}_i^{(t+1)} \defeq \vec{\vm}_i^{(t)} + f_5 \left( \vm_i^{(t)}, \Vert \vec{\vm}_i^{(t)} \Vert \right) \odot \vec{\vm}_i^{(t)} .
\end{align}
The update step implicitly builds 3-body invariants as taking the norm $\Vert \vec{\vm}_i^{(t)} \Vert = \langle  \vec{\vm}_i^{(t)}, \vec{\vm}_i^{(t)} \rangle$ computes a weighted sum of inner products $\langle \vec{\vx}_{ij}, \vec{\vx}_{ik} \rangle$ for all pairs $j, k \in \mathcal{N}_i$.

The updated scalar features are both $\fG$-invariant and $T(d)$-invariant as the only geometric information used is the relative distances, while the updated vector features are $\fG$-equivariant and $T(d)$-invariant as they aggregate $\fG$-equivariant, $T(d)$-invariant vector quantities from the neighbours.


\subsection{$\fG$-equivariant GNNs using spherical tensors}

Another example of $\fG$-equivariant GNNs is the \texttt{e3nn} framework \citep{geiger2022e3nn}, which can be used to instantiate \textbf{Tensor Field Network} \citep{thomas2018tensor}, \textbf{Cormorant} \citep{anderson2019cormorant}, \textbf{SE(3)-Transformers} \citep{fuchs2020se}, \textbf{SEGNN} \citep{brandstetter2022geometric}\footnote{We recommend the appendix of \citet{brandstetter2022geometric} for a clear and concise introduction to this class of models.}, and \textbf{MACE} \citep{batatia2022mace}.

These models use higher order spherical tensors $\tilde \vh_{i,l} \in \mathbb{R}^{2l+1 \times f}$ as node feature, starting from order $l = 0$ up to arbitrary $l = L$.
The first two orders correspond to scalar features $\vs_i$ and vector features $\vec{\vv}_i$, respectively.
The higher order tensors $\tilde \vh_{i}$ are updated via tensor products $\otimes$ of neighbourhood features $\tilde \vh_{j}$ for all $j \in \mathcal{N}_i$ with the higher order spherical harmonic representations $Y$ of the relative displacement $\frac{\vec{\vx}_{ij}}{\Vert \vec{\vx}_{ij} \Vert} = \hat{\vx}_{ij}$:
\begin{align}
\label{eq:e3nn-1}
    \tilde \vh_{i}^{(t+1)} & \defeq \tilde \vh_{i}^{(t)} + \sum_{j \in \mathcal{N}_i} Y \left( \hat{\vx}_{ij} \right) \otimes_{\vw} \tilde \vh_{j}^{(t)} ,
\end{align}
where the weights $\vw$ of the tensor product are computed via a learnt radial basis function of the relative distance, \textit{i.e.} $\vw = f \left( \Vert \vec{\vx}_{ij} \Vert \right)$.
To obtain the entry $m_3 \in \{ -l_3, \dots, +l_3 \}$ for the order-$l_3$ part of the updated higher order tensors $\tilde \vh_{i}^{(t+1)}$, we can expand the tensor product in \eqref{eq:e3nn-1} as:
\begin{align}
\label{eq:e3nn-2}
    \tilde \vh_{i,l_3 m_3}^{(t+1)} & \defeq \tilde \vh_{i,l_3 m_3}^{(t)} + \sum_{l_1 m_1, l_2 m_2}^{l_3 m_3} C_{l_1 m_1, l_2 m_2}^{l_3 m_3} \sum_{j \in \mathcal{N}_i} f_{l_1 l_2 l_3}\left( \Vert \vec{\vx}_{ij} \Vert \right) Y_{l_1}^{m_1}\left( \hat{\vx}_{ij} \right) \tilde \vh_{j,l_2 m_2}^{(t)} ,
\end{align}
where $C_{l_1 m_1, l_2 m_2}^{l_3 m_3}$ are the Clebsch-Gordan coefficients ensuring that the updated features are $\fG$-equivariant.
Notably, when restricting the tensor product to only scalars (up to $l = 0$), we obtain updates of the form similar to \eqref{eq:schnet}.
Similarly, when using only scalars and vectors (up to $l = 1$), we obtain updates of the form similar to \eqref{eq:painn-s} and \eqref{eq:painn-v}.

\textbf{MACE} \citep{batatia2022mace} provides an efficient approach to computing high $k$-body order features in the \texttt{e3nn} framework via Atomic Cluster Expansion \citep{dusson2019atomic}:
They first aggregate neighbourhood features analogous to \eqref{eq:e3nn-1} (the $A$ functions in \citet{batatia2022mace} (eq.9)) and then take $k-1$ repeated self-tensor products of these neighbourhood features. 
In our formalism, this corresponds to:
\begin{align}
\label{eq:e3nn-3}
    \tilde \vh_{i}^{(t+1)} & \defeq \underbrace {\tilde \vh_{i}^{(t+1)} \otimes_{\vw} \dots \otimes_{\vw} \tilde \vh_{i}^{(t+1)} }_\text{$k-1$ times} \ ,
\end{align}
This approach saves the effort of having to symmetrise or generate all $k$-tuples in more standard many-body expansions \textit{s.a.} DimeNet \citep{Gasteiger2020directional}. 
As an analogy, \eqref{eq:e3nn-3} amounts to calculating the product $(a + b + \dots)^k$, which implicitly includes terms such as $a^l b^{k-l}$, instead of calculating each of the $a^l b^{k-l}$ terms individually.

\section{Details on Synthetic Experiments}
\label{app:setup}

\textbf{Architectures. }
The \texttt{geometric-gnn-dojo}\footnote{Code available on GitHub: \url{https://github.com/chaitjo/geometric-gnn-dojo}} provides unified implementations of several popular geometric GNN architectures characterised by GWL:
\begin{itemize}
    \item $\fG$-invariant GNNs: SchNet \citep{schutt2018schnet}, DimeNet \citep{Gasteiger2020directional}, and SphereNet \citep{liu2022spherical};
    \item $\fG$-equivariant GNNs using cartesian vectors: E-GNN \citep{satorras2021n} and GVP-GNN \citep{jing2020learning};
    \item $\fG$-equivariant GNNs using spherical tensors: TFN \citep{thomas2018tensor} and MACE \citep{batatia2022mace}.
\end{itemize}

\textbf{Tasks. }
We design synthetic experiments to highlight practical challenges in building expressive geometric GNNs:
\begin{itemize}
    \item \emph{Distinguishing $k$-chains}, which test a model's ability to propagate geometric information non-locally and demonstrate geometric oversquashing with increased depth; see Table \ref{tab:kchains}.
    \item \emph{Rotationally symmetric structures}, which test a layer's ability to identify neighbourhood orientation and highlight the utility of higher order tensors in $\fG$-equivariant GNNs; see Table \ref{tab:rotsym}.
    \item \emph{Counterexamples from \citet{pozdnyakov2020incompleteness}}, which test a layer's ability to create distinguishing fingerprints for local neighbourhoods and highlight the need for higher body order of scalarisation; see Table \ref{tab:ihash}.
\end{itemize}

\textbf{Hyperparameters. }
For SchNet and DimeNet, we use the implementation from PyTorch Geometric \citep{Fey/Lenssen/2019}.
For SphereNet, E-GNN, GVP-GNN, and MACE, we adapt implementations from the respective authors.
Our TFN implementation is based on \texttt{\small e3nn} \citep{geiger2022e3nn}, and we also re-implement MACE by incorporating the \texttt{\small EquivariantProductBasisBlock} from its authors into our TFN layer.
We set scalar feature channels to 128 for SchNet, DimeNet, SphereNet, and E-GNN.
We set scalar/vector/tensor feature channels to 64 for GVP-GNN, TFN, MACE.
TFN and MACE use order $L=2$ tensors by default. 
MACE uses local body order 4 by default.
We train all models for 100 epochs using the Adam optimiser, with an initial learning rate $1e-4$, which we reduce by a factor of 0.9 and patience of 25 epochs when the performance plateaus.
All results are averaged across 10 random seeds.


\begin{table}[h!]
    \centering
     \resizebox{0.6\linewidth}{!}{
    \begin{tabular}{clccc}
        \toprule
        & & \multicolumn{3}{c}{\textbf{Counterexample from \citet{pozdnyakov2020incompleteness}}} \\
        & \multirow{2}{*}{\textbf{GNN Layer}} & 2-body & 3-body & 4-body \\
        & & & \gray{(Fig.1(b))} & \gray{(Fig.2(f))} \\
        \midrule
        \multirow{3}{*}{\rotatebox[origin=c]{90}{Inv.}} 
        & SchNet$_{\text{2-body}}$ & \cellcolor{red!10} 50.0 ± 0.0 & 50.0 ± 0.0 & 50.0 ± 0.0 \\
        & DimeNet$_{\text{3-body}}$ & \cellcolor{green!10} \textbf{100.0 ± 0.0} & \cellcolor{red!10} 50.0 ± 0.0 & 50.0 ± 0.0 \\
        & SphereNet$_{\text{4-body}}$ & \cellcolor{green!10} \textbf{100.0 ± 0.0} & \cellcolor{green!10} \textbf{100.0 ± 0.0} & \cellcolor{red!10} 50.0 ± 0.0 \\
        \midrule
        \multirow{6}{*}{\rotatebox[origin=c]{90}{$O(3)$-Equiv.}} 
        & E-GNN$_{\text{2-body}}$ & \cellcolor{red!10} 50.0 ± 0.0 & 50.0 ± 0.0 & 50.0 ± 0.0 \\
        & GVP-GNN$_{\text{3-body}}$ & \cellcolor{green!10} \textbf{100.0 ± 0.0} & \cellcolor{red!10} 50.0 ± 0.0 & 50.0 ± 0.0 \\
        & TFN$_{\text{2-body}}$ & \cellcolor{red!10} 50.0 ± 0.0 & 50.0 ± 0.0 & 50.0 ± 0.0 \\
        & MACE$_{\text{3-body}}$ & \cellcolor{green!10} \textbf{100.0 ± 0.0} & 50.0 ± 0.0 & 50.0 ± 0.0 \\
        & MACE$_{\text{4-body}}$ & \cellcolor{green!10} \textbf{100.0 ± 0.0} & \cellcolor{green!10} \textbf{100.0 ± 0.0} & 50.0 ± 0.0 \\
        & MACE$_{\text{5-body}}$ & \cellcolor{green!10} \textbf{100.0 ± 0.0} & \cellcolor{green!10} \textbf{100.0 ± 0.0} & \cellcolor{green!10} \textbf{100.0 ± 0.0} \\
        \bottomrule
    \end{tabular}
    }
    \caption{\textit{Counterexamples from \citet{pozdnyakov2020incompleteness}.}
    This task evaluates single layer geometric GNNs at distinguishing counterexample structures that are indistinguishable using $k$-body scalarisation.
    \textbf{Geometric GNN layers with body order $k$ cannot distinguish the corresponding counterexample.}
    The 3-body counterexample is from Fig.1(b) \citep{pozdnyakov2020incompleteness}, 4-body is from Fig.2(f) \citep{pozdnyakov2020incompleteness}, and 2-body is based on the two local neighbourhoods in our running example.
    }
    \label{tab:ihash}
\end{table}

\section{Geometric GNN Design Space Proofs}
\label{app:design-space}

\subsection{Role of depth (Section \ref{sec:designspace:depth})}

The following results are a consequence of the construction of GWL as well as the definitions of $k$-hop distinct and $k$-hop identical geometric graphs.
Note that $k$-hop distinct geometric graphs are also $(k+1)$-hop distinct. Similarly, $k$-hop identical geometric graphs are also $(k-1)$-hop identical, but not necessarily $(k+1)$-hop distinct.

Given two distinct neighbourhoods $\gN_1$ and $\gN_2$, the $\fG$-orbits of the corresponding geometric multisets $\vg_1$ and $\vg_2$ are mutually exclusive, \textit{i.e.} $\gO_{\fG}(\vg_1) \cap \gO_{\fG}(\vg_2) \equiv \emptyset$. By the properties of $\textsc{I-Hash}$ this implies $c_1 \neq c_2 $.
Conversely, if $\gN_1$ and $\gN_2$ were identical up to group actions, their $\fG$-orbits would overlap, \textit{i.e.} $\vg_1 = \fg \ \vg_2$ for some $\fg \in \fG$ and $\gO_{\fG}(\vg_1) = \gO_{\fG}(\vg_2) \ \Rightarrow \ c_1 = c_2$.

\gwlcanprop*
\begin{proof}[Proof of Proposition \ref{prop:gwl-can}]
    The $k$-th iteration of GWL identifies the $\fG$-orbit of the $k$-hop subgraph $\gN_i^{(k)}$ at each node $i$ via the geometric multiset $\vg_i^{(k)}$.
    $\mathcal{G}_1$ and $\mathcal{G}_2$ being $k$-hop distinct implies that there exists some bijection $b$ and some node $i \in \mathcal{V}_1, b(i) \in \mathcal{V}_2$ such that the corresponding $k$-hop subgraphs $\gN_i^{(k)}$ and $\gN_{b(i)}^{(k)}$ are distinct.
    Thus, the $\fG$-orbits of the corresponding geometric multisets $\vg_i^{(k)}$ and $\vg_{b(i)}^{(k)}$ are mutually exclusive, \textit{i.e.} $\gO_{\fG}(\vg_i^{(k)}) \cap \gO_{\fG}(\vg_{b(i)}^{(k)}) \equiv \emptyset \ \Rightarrow \ c_i^{(k)} \neq c_{b(i)}^{(k)} $.
    Thus, $k$ iterations of GWL are sufficient to distinguish $\mathcal{G}_1$ and $\mathcal{G}_2$.
\end{proof}

\gwlcannotprop*
\begin{proof}[Proof of Proposition \ref{prop:gwl-cannot}]
    The $k$-th iteration of GWL identifies the $\fG$-orbit of the $k$-hop subgraph $\gN_i^{(k)}$ at each node $i$ via the geometric multiset $\vg_i^{(k)}$.
    $\mathcal{G}_1$ and $\mathcal{G}_2$ being $k$-hop identical implies that for all bijections $b$ and all nodes $i \in \mathcal{V}_1, b(i) \in \mathcal{V}_2$, the corresponding $k$-hop subgraphs $\gN_i^{(k)}$ and $\gN_{b(i)}^{(k)}$ are identical up to group actions.
    Thus, the $\fG$-orbits of the corresponding geometric multisets $\vg_i^{(k)}$ and $\vg_{b(i)}^{(k)}$ overlap, \textit{i.e.} $\gO_{\fG}(\vg_i^{(k)}) = \gO_{\fG}(\vg_{b(i)}^{(k)}) \ \Rightarrow \ c_i^{(k)} = c_{b(i)}^{(k)}$.
    Thus, up to $k$ iterations of GWL cannot distinguish $\mathcal{G}_1$ and $\mathcal{G}_2$.
\end{proof}

\igwlcanprop*
\begin{proof}[Proof of Proposition \ref{prop:igwl-can}]
    Each iteration of IGWL identifies the $\fG$-orbit of the $1$-hop local neighbourhood $\gN_i^{(k=1)}$ at each node $i$.
    $\mathcal{G}_1$ and $\mathcal{G}_2$ being $1$-hop distinct implies that there exists some bijection $b$ and some node $i \in \mathcal{V}_1, b(i) \in \mathcal{V}_2$ such that the corresponding $1$-hop local neighbourhoods $\gN_i^{(1)}$ and $\gN_{b(i)}^{(1)}$ are distinct.
    Thus, the $\fG$-orbits of the corresponding geometric multisets $\vg_i^{(1)}$ and $\vg_{b(i)}^{(1)}$ are mutually exclusive, \textit{i.e.} $\gO_{\fG}(\vg_i^{(1)}) \cap \gO_{\fG}(\vg_{b(i)}^{(1)}) \equiv \emptyset \ \Rightarrow \ c_i^{(1)} \neq c_{b(i)}^{(1)} $.
    Thus, 1 iteration of IGWL is sufficient to distinguish $\mathcal{G}_1$ and $\mathcal{G}_2$.
\end{proof}



\igwlcannotprop*
\begin{proof}[Proof of Proposition \ref{prop:igwl-cannot}]
    Each iteration of IGWL identifies the $\fG$-orbit of the $1$-hop local neighbourhood $\gN_i^{(k=1)}$ at each node $i$, but cannot identify $\fG$-orbits beyond $1$-hop by the construction of IGWL as no geometric information is propagated.
    $\mathcal{G}_1$ and $\mathcal{G}_2$ being $1$-hop identical implies that for all bijections $b$ and all nodes $i \in \mathcal{V}_1, b(i) \in \mathcal{V}_2$, the corresponding $1$-hop local neighbourhoods $\gN_i^{(k)}$ and $\gN_{b(i)}^{(k)}$ are identical up to group actions.
    Thus, the $\fG$-orbits of the corresponding geometric multisets $\vg_i^{(1)}$ and $\vg_{b(i)}^{(1)}$ overlap, \textit{i.e.} $\gO_{\fG}(\vg_i^{(1)}) = \gO_{\fG}(\vg_{b(i)}^{(1)}) \ \Rightarrow \ c_i^{(k)} = c_{b(i)}^{(k)}$.
    Thus, any number of IGWL iterations cannot distinguish $\mathcal{G}_1$ and $\mathcal{G}_2$.
\end{proof}

\nonisoprop*
\begin{proof}[Proof of Proposition \ref{prop:noniso}]
    We assume that a geometric graph $\mathcal{G} = ( \mA, \mS, \vec{\mV}, \vec{\mX} )$ is constructed from a point cloud $( \mS, \vec{\mV}, \vec{\mX} )$ using a predetermined radial cutoff $r$.
    Thus, the adjacency matrix is defined as $a_{ij} = 1 \text{ if } \Vert \vec{\vx}_i-\vec{\vx}_j \Vert_2 \leq r$, or $0$ otherwise, for all $a_{ij} \in \mA$.
    Such construction procedures are conventional for geometric graphs in biochemistry and material science.

    Given geometric graphs $\mathcal{G}_1$ and $\mathcal{G}_2$ where the underlying attributed graphs are non-isomorphic, identify $k_{\text{Max}}$ the maximum of the graph diameters of $\mathcal{G}_1$ and $\mathcal{G}_2$, and chose any arbitrary nodes $i \in \gV_1, j \in \gV_2$.
    We can define the $k_{\text{Max}}$-hop subgraphs $\mathcal{N}_i^{(k_{\text{Max}})}$ and $\mathcal{N}_j^{(k_{\text{Max}})}$ at $i$ and $j$, respectively.
    Thus, $\gN_i^{(k_{\text{Max}})} = \gV_1$ for all $i \in \gV_1$, and $\gN_j^{(k_{\text{Max}})} = \gV_2$ for all $j \in \gV_2$.
    Due to the assumed construction procedure of geometric graphs, $\mathcal{N}_i^{(k_{\text{Max}})}$ and $\mathcal{N}_j^{(k_{\text{Max}})}$ must be distinct.
    Otherwise, if $\mathcal{N}_i^{(k_{\text{Max}})}$ and $\mathcal{N}_j^{(k_{\text{Max}})}$ were identical up to group actions, the sets $( \mS_1, \vec{\mV}_1, \vec{\mX}_1 )$ and $( \mS_2, \vec{\mV}_2, \vec{\mX}_2 )$ would have yielded isomorphic graphs.

    The $k_{\text{Max}}$-th iteration of GWL identifies the $\fG$-orbit of the $k_{\text{Max}}$-hop subgraph $\gN_i^{(k_{\text{Max}})}$ at each node $i$ via the geometric multiset $\vg_i^{(k_{\text{Max}})}$.
    As $\mathcal{N}_i^{(k_{\text{Max}})}$ and $\mathcal{N}_j^{(k_{\text{Max}})}$ are distinct for any arbitrary nodes $i \in \gV_1, j \in \gV_2$, the $\fG$-orbits of the corresponding geometric multisets $\vg_i^{(k_{\text{Max}})}$ and $\vg_{j}^{(k_{\text{Max}})}$ are mutually exclusive, \textit{i.e.} $\gO_{\fG}(\vg_i^{(k_{\text{Max}})}) \cap \gO_{\fG}(\vg_{j}^{(k_{\text{Max}})}) \equiv \emptyset \ \Rightarrow \ c_i^{(k_{\text{Max}})} \neq c_{j}^{(k_{\text{Max}})} $.
    Thus, $k_{\text{Max}}$ iterations of GWL are sufficient to distinguish $\mathcal{G}_1$ and $\mathcal{G}_2$.
\end{proof}


\subsection{Limitations of invariant message passing (Section \ref{sec:designspace:inv-equiv})}

\stricttheo*
\begin{proof}[Proof of Theorem \ref{theo:strict}]

    Firstly, we can show that the GWL class contains IGWL if GWL can learn the identity when updating $\vg_i$ for all $i \in \mathcal{V}$, \textit{i.e.} $\vg_i^{(t)} = \vg_i^{(t-1)} = {\vg}_i^{(0)} \equiv (\vs_i, \vec{\vv}_i)$.
    Thus, GWL is at least as powerful as IGWL, which does not update $\vg_i$.

    Secondly, to show that GWL is strictly more powerful than IGWL, it suffices to show that there exist a pair of geometric graphs that can be distinguished by GWL but not by IGWL.
    We may consider any $k$-hop distinct geometric graphs for $k > 1$, where the underlying attributed graphs are isomorphic.
    Proposition \ref{prop:gwl-can} states that GWL can distinguish any such graphs, while Proposition \ref{prop:igwl-cannot} states that IGWL cannot distinguish them.
    An example is the pair of graphs in Figures \ref{fig:gwl} and \ref{fig:igwl}.
\end{proof}


\globalgeomprop*
\begin{proof}[Proof of Proposition \ref{prop:globalgeom}]
    Following \citet{garg2020generalization}, we say that a class of models \textit{decides} a geometric graph property if there exists a model belonging to this class such that for any two geometric graphs that differ in the property, the model is able to distinguish the two geometric graphs.


\begin{figure}[t!]
    \centering
    \includegraphics[width=0.65\linewidth]{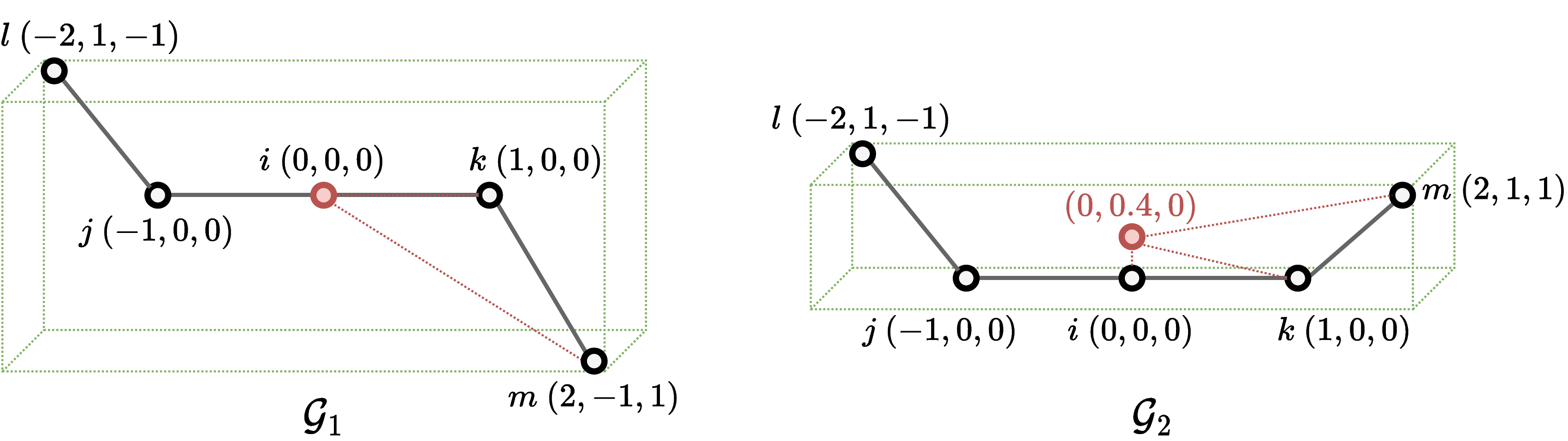}
    \caption{
        Two geometric graphs for which IGWL and $\fG$-invariant GNNs cannot distinguish their  perimeter, surface area, volume of the bounding box/sphere, distance from the centroid, and dihedral angles.
        The centroid is denoted by a red point and distances from it are denoted by dotted red lines.
        The bounding box enclosing the geometric graph is denoted by the dotted green lines.}
    \label{fig:globalgeom}
\end{figure}


    In Figure \ref{fig:globalgeom}, we provide an example of two geometric graphs that demonstrate the proposition.
    $\mathcal{G}_1$ and $\mathcal{G}_2$ differ in the following geometric graph properties:
    \begin{itemize}
        \item Perimeter, surface area, and volume of the bounding box enclosing the geometric graph\footnote{The same result applies for the bounding sphere, not shown in the figure.}: (32 units, 40 units$^2$, 16 units$^3$) vs. (28 units, 24 units$^2$, 8 units$^3$).
        \item Multiset of distances from the centroid or centre of mass: $\{ 0.00, 1.00, 1.00, 2.45, 2.45 \}$ vs. $\{ 0.40, 1.08, 1.08, 2.32, 2.32  \}$.
        \item Dihedral angles: $\angle (ljkm) = \frac{ (\vec{\vx}_{jk} \times \vec{\vx}_{lj}) \cdot (\vec{\vx}_{jk} \times \vec{\vx}_{mk}) }{ \vert\vec{\vx}_{jk} \times \vec{\vx}_{lj}\vert \vert\vec{\vx}_{jk} \times \vec{\vx}_{mk}\vert }$ are clearly different for the two graphs.
    \end{itemize}
    However, according to Proposition \ref{prop:igwl-cannot} and Theorem \ref{theo:invupper}, both IGWL and $\fG$-invariant GNNs cannot distinguish these two geometric graphs, and therefore, cannot decide all these properties.


\begin{figure*}[t!]
    \centering
    \includegraphics[width=0.7\linewidth]{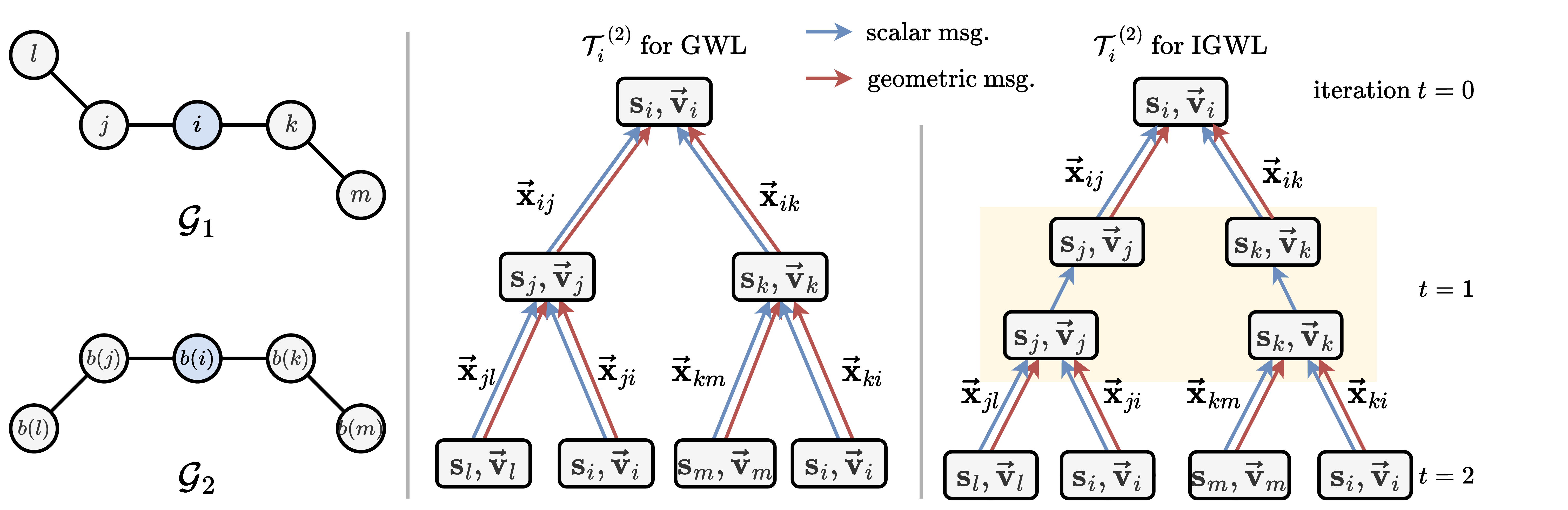}
    \caption{\textbf{Geometric Computation Trees for GWL and IGWL.}
        Unlike GWL, geometric orientation information cannot flow from the leaves to the root in IGWL, restricting its expressive power.
        IGWL cannot distinguish $\mathcal{G}_1$ and $\mathcal{G}_2$ as all $1$-hop neighbourhoods are computationally identical.}
    \label{fig:trees}
\end{figure*}
    

    We can also show this via a geometric version of computation trees \citep{garg2020generalization}, for any number of IGWL or $\fG$-invariant GNN iterations, as illustrated in Figure \ref{fig:trees}.
    A computation tree $\mathcal{T}_i^{(t)}$ represents the maximum information contained in GWL/IGWL colours or GNN features for node $i$ at iteration $t$ by an `unrolling' of the message passing procedure.
    GWL, IGWL, and the corresponding classes of GNNs can be intuitively understood as colouring geometric computation trees.
    
    Geometric computation trees are constructed recursively: 
    $\mathcal{T}_i^{(0)} = (\vs_i, \vec{\vv}_i)$ for all $i \in \mathcal{V}$.
    For $t > 0$, we start with a root node $(\vs_i, \vec{\vv}_i)$ and add a child subtree $\mathcal{T}_j^{(t-1)}$ for all $j \in \mathcal{N}_i$ along with the relative position $\vec{\vx}_{ij}$ along the edge.
    To obtain the root node's embedding or colour, both scalar and geometric information is propagated from the leaves up to the root.
    Thus, if two nodes have identical geometric computation trees, they will be mapped to the same node embedding or colour.
    
    Critically, geometric orientation information cannot flow from one level to another in the computation trees for IGWL and $\fG$-invariant GNNs, as they only update scalar information.
    In the recursive construction procedure, we must insert a connector node $(\vs_j, \vec{\vv}_j)$ before adding the child subtree $\mathcal{T}_j^{(t-1)}$ for all $j \in \mathcal{N}_i$ and prevent geometric information propagation between them.

    Following the construction procedure for the geometric graphs in Figure \ref{fig:globalgeom}, we observe that the IGWL computation trees of any pair of isomorphic nodes are identical, as all 1-hop neighbourhoods are computationally identical.
    Therefore, the set of node colours or node scalar features will also be identical, which implies that $\mathcal{G}_1$ and $\mathcal{G}_2$ cannot be distinguished.
\end{proof}


\limitoffullprop*
\begin{proof}[Proof of Proposition \ref{prop:limit-of-full}]
    We will prove by contradiction.
    Assume that there exist a pair of fully connected geometric graphs $\mathcal{G}_1$ and $\mathcal{G}_2$ which GWL can distinguish, but IGWL cannot.

    If the underlying attributed graphs of $\mathcal{G}_1$ and $\mathcal{G}_2$ are isomorphic, by Proposition \ref{prop:gwl-can} and Proposition \ref{prop:igwl-cannot}, $\mathcal{G}_1$ and $\mathcal{G}_2$ are $1$-hop identical but $k$-hop distinct for some $k > 1$.
    For all bijections $b$ and all nodes $i \in \mathcal{V}_1, b(i) \in \mathcal{V}_2$, the local neighbourhoods $\mathcal{N}_i^{(1)}$ and $\mathcal{N}_{b(i)}^{(1)}$ are identical up to group actions, and $\gO_{\fG}(\vg_i^{(1)}) = \gO_{\fG}(\vg_{b(i)}^{(1)}) \ \Rightarrow \ c_i^{(1)} = c_{b(i)}^{(1)}$.
    Additionally, there exists some bijection $b$ and some nodes $i \in \mathcal{V}_1, b(i) \in \mathcal{V}_2$ such that the $k$-hop subgraphs $\mathcal{N}_i^{(k)}$ and $\mathcal{N}_{b(i)}^{(k)}$ are distinct, and $\gO_{\fG}(\vg_i^{(k)}) \cap \gO_{\fG}(\vg_{b(i)}^{(k)}) \equiv \emptyset \ \Rightarrow \ c_i^{(k)} \neq c_{b(i)}^{(k)} $.
    However, as $\mathcal{G}_1$ and $\mathcal{G}_2$ are fully connected, for any $k$, $\mathcal{N}_i^{(1)} = \mathcal{N}_i^{(k)}$ and $\mathcal{N}_{b(i)}^{(1)} = \mathcal{N}_{b(i)}^{(k)}$ are identical up to group actions.
    Thus, $\gO_{\fG}(\vg_i^{(1)}) = \gO_{\fG}(\vg_i^{(k)}) = \gO_{\fG}(\vg_{b(i)}^{(1)}) = \gO_{\fG}(\vg_{b(i)}^{(k)}) \ \Rightarrow \ c_i^{(1)} = c_i^{(k)} = c_{b(i)}^{(k)} = c_{b(i)}^{(k)}$.
    This is a contradiction.

    If $\mathcal{G}_1$ and $\mathcal{G}_2$ are non-isomorphic and fully connected, for any arbitrary $i \in \gV_1, j \in \gV_2$ and any $k$-hop neighbourhood, we know that $\gN_i^{(1)} = \gN_i^{(k)}$ and $\gN_j^{(1)} = \gN_j^{(k)}$.
    Thus, a single iteration of GWL and IGWL identify the same $\fG$-orbits and assign the same node colours, \textit{i.e.} $\gO_{\fG}(\vg_i^{(1)}) = \gO_{\fG}(\vg_i^{(k)}) \ \Rightarrow \ c_i^{(1)} = c_i^{(k)}$ and $\gO_{\fG}(\vg_j^{(1)}) = \gO_{\fG}(\vg_j^{(k)})\ \Rightarrow \ c_j^{(1)} = c_j^{(k)}$.
    This is a contradiction.
\end{proof}


\subsection{Role of scalarisation body order (Section \ref{sec:designspace:scalarisation})}

\bodyorderprop*
\begin{proof}[Proof of Proposition \ref{prop:bodyorder}]
    As $m$ is the maximum cardinality of all local neighbourhoods $\gN_i$ under consideration, any distinct neighbourhoods $\gN_1$ and $\gN_2$ must have distinct multisets of $m$-body scalars.
    As $\textsc{I-Hash}_{(m)}$ computes scalars involving up to $m$ nodes, it will be able to distinguish any such $\gN_1$ and $\gN_2$.
    Thus, $\textsc{I-Hash}_{(m)}$ is $\fG$-orbit injective.
\end{proof}

\igwlhierarchyprop*
\begin{proof}[Proof of Proposition \ref{prop:igwl-hierarchy}]
    By construction, ${\textsc{I-Hash}}_{(k)}$ computes $\fG$-invariant scalars from all possible tuples of up to $k$ nodes formed by the elements of a neighbourhood and the central node.
    Thus, the ${\textsc{I-Hash}}_{(k)}$ class contains ${\textsc{I-Hash}}_{(k-1)}$, and ${\textsc{I-Hash}}_{(k)}$
    is at least as powerful as ${\textsc{I-Hash}}_{(k-1)}$.
    Thus, the corresponding test ${\text{IGWL}}_{(k)}$ is at least as powerful as ${\text{IGWL}}_{(k-1)}$.

    Secondly, to show that ${\text{IGWL}}_{(k)}$ is strictly more powerful than ${\text{IGWL}}_{(k-1)}$ for $k \leq 5$, it suffices to show that there exist a pair of geometric neighbourhoods that can be distinguished by ${\text{IGWL}}_{(k)}$ but not by ${\text{IGWL}}_{(k-1)}$:
    \begin{itemize}
        \item For $k = 3$ and $\fG = O(3)~\text{or}~SO(3)$, for the local neighbourhood from Figure 1 in \cite{schutt2021equivariantmp}, two configurations with different angles between the neighbouring nodes can be distinguished by ${\text{IGWL}}_{(3)}$ but not by ${\text{IGWL}}_{(2)}$.
        \item For $k = 4$ and $\fG = O(3)~\text{or}~SO(3)$, the pair of local neighbourhoods from Figure 1 in \cite{pozdnyakov2020incompleteness} can be distinguished by ${\text{IGWL}}_{(4)}$ but not by ${\text{IGWL}}_{(3)}$.
        \item For $k = 5$ and $\fG = O(3)$, the pair of local neighbourhoods from Figure 2(e) in \cite{pozdnyakov2020incompleteness} can be distinguished by ${\text{IGWL}}_{(5)}$ but not by ${\text{IGWL}}_{(4)}$.
        \item For $k = 5$ and $\fG = SO(3)$, the pair of local neighbourhoods from Figure 2(f) in \cite{pozdnyakov2020incompleteness} can be distinguished by ${\text{IGWL}}_{(5)}$ but not by ${\text{IGWL}}_{(4)}$.
    \end{itemize}
\end{proof}

\onewligwlprop*
\begin{proof}[Proof of Proposition \ref{prop:one-wl-igwl}]
    Let $c$ and $k$ the colours produced by ${\text{IGWL}}_{(2)}$ and WL, respectively, and let $i$ and $j$ be two nodes belonging to any two graphs like in the statement of the result. We prove the statement inductively.

    Clearly, $c^{(0)}_i = k^{(0)}_i$ for all nodes $i$ and $c^{(0)}_i = c^{(0)}_j$ if and only if $k^{(0)}_i = k^{(0)}_j$. Now, assume that the statement holds for iteration $t$. That is $c^{(t)}_i = c^{(t)}_j$ if and only if $k^{(t)}_i = k^{(t)}_j$ holds for all $i$. Note that $c^{(t+1)}_i = c^{(t+1)}_j$ if and only if $c^{(t)}_i = c^{(t)}_j$ and $\ldblbrace (c^{(t)}_p, \norm{\vec{\vx}_{ip}}) \mid p \in \gN_i \rdblbrace = \ldblbrace (c^{(t)}_p, \norm{\vec{\vx}_{jp}}) \mid p \in \gN_j \rdblbrace$, since the norm of the relative vectors is the only injective invariant that ${\text{IGWL}}_{(2)}$ can compute (up to a scaling). Since all the norms are equal, by the induction hypothesis, this is equivalent to $k^{(t)}_i = k^{(t)}_j$ and $\ldblbrace k^{(t)}_p \mid p \in \gN_i \rdblbrace = \ldblbrace k^{(t)} \mid p \in \gN_j \rdblbrace$. Therefore, this is equivalent to $k^{(t+1)}_i = k^{(t+1)}_j$
\end{proof}



\section{Proofs for equivalence between GWL and Geometric GNNs (Section \ref{sec:gwl:equivalence})}
\label{app:gwl-equivalence}

Our proofs adapt the techniques used in \citet{xu2018how, morris2019weisfeiler} for connecting WL with GNNs.
Note that we omit including the relative position vectors $\vec{\vx}_{ij}$ in GWL and geometric GNN updates for brevity, as relative positions vectors can be merged into the vector features.

\uppertheo*
\begin{proof}[\textbf{Proof of Theorem \ref{theo:upper}}]

    Consider two geometric graphs $\mathcal{G}$ and $\mathcal{H}$.
    The theorem implies that if the GNN graph-level readout outputs $f(\mathcal{G}) \neq f(\mathcal{H})$, then the GWL test will always determine $\mathcal{G}$ and $\mathcal{H}$ to be non-isomorphic, \textit{i.e.} $\mathcal{G} \neq \mathcal{H}$.

    We will prove by contradiction.
    Suppose after $T$ iterations, a GNN graph-level readout outputs $f(\mathcal{G}) \neq f(\mathcal{H})$, but the GWL test cannot decide $\mathcal{G}$ and $\mathcal{H}$ are non-isomorphic, \textit{i.e.} $\mathcal{G}$ and $\mathcal{H}$ always have the same collection of node colours for iterations $0$ to $T$.
    Thus, for iteration $t$ and $t+1$ for any $t = 0 \dots T-1$, $\mathcal{G}$ and $\mathcal{H}$ have the same collection of node colours $\{ c_i^{(t)} \}$ as well as the same collection of neighbourhood geometric multisets $\left\lbrace (c_i^{(t)}, \vg_i^{(t)}) \ , \  \ldblbrace  (c_j^{(t)}, \vg_j^{(t)}) \mid j \in \mathcal{N}_i \rdblbrace \right\rbrace$ up to group actions.
    Otherwise, the GWL test would have produced different node colours at iteration $t+1$ for $\mathcal{G}$ and $\mathcal{H}$ as different geometric multisets get unique new colours.

    We will show that on the same graph for nodes $i$ and $k$, if $(c_{i}^{(t)}, \vg_i^{(t)}) = (c_{k}^{(t)}, \fg \cdot \vg_{k}^{(t)})$, we always have GNN features $(\vs_i^{(t)}, \vec{\vv}_i^{(t)}) = (\vs_k^{(t)}, \mQ_{\fg} \vec{\vv}_k^{(t)})$ for any iteration $t$.
    This holds for $t = 0$ because GWL and the GNN start with the same initialisation.
    Suppose this holds for iteration $t$.
    At iteration $t+1$, if for any $i$ and $k$, $(c_{i}^{(t+1)}, \vg_i^{(t+1)}) = (c_{k}^{(t+1)}, \fg \cdot \vg_{k}^{(t+1)})$, then:
    \begin{align}
        \left\lbrace (c_i^{(t)}, \vg_i^{(t)}) \ , \ \ldblbrace (c_j^{(t)}, \vg_j^{(t)}) \mid j \in \mathcal{N}_i \rdblbrace \right\rbrace & = \left\lbrace (c_k^{(t)}, \fg \cdot \vg_k^{(t)}) \ , \ \ldblbrace (c_j^{(t)}, \fg \cdot \vg_j^{(t)}) \mid j \in \mathcal{N}_k \rdblbrace \right\rbrace
    \end{align}
    By our assumption on iteration $t$,
    \begin{align}
        \left\lbrace (\vs_i^{(t)}, \vec{\vv}_i^{(t)}) \ , \ \ldblbrace (\vs_j^{(t)}, \vec{\vv}_j^{(t)}) \mid j \in \mathcal{N}_i \rdblbrace \right\rbrace & = \left\lbrace (\vs_k^{(t)}, \mQ_{\fg} \vec{\vv}_k^{(t)}) \ , \ \ldblbrace (\vs_j^{(t)}, \mQ_{\fg} \vec{\vv}_j^{(t)}) \mid j \in \mathcal{N}_k \rdblbrace \right\rbrace
    \end{align}
    As the same aggregate and update operations are applied at each node within the GNN, the same inputs, \textit{i.e.} neighbourhood features, are mapped to the same output. Thus, $(\vs_i^{(t+1)}, \vec{\vv}_i^{(t+1)}) = (\vs_k^{(t+1)}, \mQ_{\fg} \vec{\vv}_k^{(t+1)})$.
    By induction, if $(c_{i}^{(t)}, \vg_i^{(t)}) = (c_{k}^{(t)}, \fg \cdot \vg_{k}^{(t)})$, we always have GNN node features $(\vs_i^{(t)}, \vec{\vv}_i^{(t)}) = (\vs_k^{(t)}, \mQ_{\fg} \vec{\vv}_k^{(t)})$ for any iteration $t$.
    This creates valid mappings $\phi_s, \phi_v$ such that $\vs_i^{(t)} = \phi_s(c_i^{(t)})$ and $\vec{\vv}_i^{(t)} = \phi_v(c_i^{(t)}, \vg_i^{(t)})$ for any $i \in \mathcal{V}$.

    Thus, if $\mathcal{G}$ and $\mathcal{H}$ have the same collection of node colours and geometric multisets, then $\mathcal{G}$ and $\mathcal{H}$ also have the same collection of GNN neighbourhood features
    \begin{align*}
        \left\lbrace (\vs_i^{(t)}, \vec{\vv}_i^{(t)}) \ , \ \ldblbrace (\vs_j^{(t)}, \vec{\vv}_j^{(t)}) \mid j \in \mathcal{N}_i \rdblbrace \right\rbrace & = \left\lbrace (\phi_s(c_i^{(t)}), \phi_v(c_i^{(t)}, \vg_i^{(t)}) ) \ , \ \ldblbrace (\phi_s(c_j^{(t)}), \phi_v(c_i^{(t)}, \vg_i^{(t)}) ) \mid j \in \mathcal{N}_i \rdblbrace \right\rbrace
    \end{align*}
    Thus, the GNN will output the same collection of node scalar features $\{ \vs_i^{(T)} \}$ for $\mathcal{G}$ and $\mathcal{H}$ and the permutation-invariant graph-level readout will output $f(\mathcal{G}) = f(\mathcal{H})$.
    This is a contradiction.
\end{proof}

Similarly, $\fG$-invariant GNNs can be at most as powerful as IGWL.
\begin{restatable}{theorem}{invuppertheo}
    \label{theo:invupper}
    Any pair of geometric graphs distinguishable by a $\fG$-invariant GNN is also distinguishable by IGWL.
\end{restatable}
\begin{proof}
The proof follows similarly to the proof for Theorem~\ref{theo:upper}. 
\end{proof}


\conditionstheo*

\begin{proof}[\textbf{Proof of Theorem \ref{theo:conditions}}]

    Consider a GNN where the conditions hold.
    We will show that, with a sufficient number of iterations $t$, the output of this GNN is equivalent to GWL, \textit{i.e.} $\vs^{(t)} \equiv c^{(t)}$.

    Let $\mathcal{G}$ and $\mathcal{H}$ be any geometric graphs which the GWL test decides as non-isomorphic at iteration $T$.
    Because the graph-level readout function is injective, \textit{i.e.} it maps distinct multiset of node scalar features into unique embeddings, it suffices to show that the GNN's neighbourhood aggregation process, with sufficient iterations, embeds $\mathcal{G}$ and $\mathcal{H}$ into different multisets of node features.

    For this proof, we replace $\fG$-orbit injective functions with injective functions over the equivalence class generated by the actions of $\fG$.
    Thus, all elements belonging to the same $\fG$-orbit will first be mapped to the same representative of the equivalence class, denoted by the square brackets $\left[ \dots \right]$, followed by an injective map.
    The result is $\fG$-orbit injective.

    Let us assume the GNN updates node scalar and vector features as:
    \begin{align}
        \vs_i^{(t)}       & = \textsc{Upd}_s \left( \left[  (\vs_i^{(t-1)}, \vec{\vv}_i^{(t-1)}) \ , \ \textsc{Agg} \left( \ldblbrace (\vs_i^{(t-1)}, \vs_j^{(t-1)}, \vec{\vv}_i^{(t-1)},  \vec{\vv}_j^{(t-1)}) \mid j \in \mathcal{N}_i \rdblbrace \right) \right]  \right) \\
        \vec{\vv}_i^{(t)} & = \textsc{Upd}_v \left( (\vs_i^{(t-1)}, \vec{\vv}_i^{(t-1)}) \ , \ \textsc{Agg} \left( \ldblbrace (\vs_i^{(t-1)}, \vs_j^{(t-1)}, \vec{\vv}_i^{(t-1)},  \vec{\vv}_j^{(t-1)}) \mid j \in \mathcal{N}_i \rdblbrace \right) \right)
    \end{align}
    with the aggregation function $\textsc{Agg}$ being $\fG$-equivariant and injective, the scalar update function $\textsc{Upd}_s$ being $\fG$-invariant and injective, and the vector update function $\textsc{Upd}_{v}$ being $\fG$-equivariant and injective.

    The GWL test updates the node colour $c_i^{(t)}$ and geometric multiset $\vg_i^{(t)}$ as:
    \begin{align}
        c_i^{(t)}   & = h_{s} \left( \left[ (c_i^{(t-1)}, \vg_i^{(t-1)}) \ , \  \ldblbrace  (c_j^{(t-1)}, \vg_j^{(t-1)}) \mid j \in \mathcal{N}_i \rdblbrace  \right] \right), \\
        \vg_i^{(t)} & = h_{v} \left( (c_i^{(t-1)}, \vg_i^{(t-1)}) \ , \  \ldblbrace  (c_j^{(t-1)}, \vg_j^{(t-1)}) \mid j \in \mathcal{N}_i \rdblbrace \right),
    \end{align}
    where $h_{s}$ is a $\fG$-invariant and injective map, and $h_{v}$ is a $\fG$-equivariant and injective operation (e.g. in \eqref{eq:gwl_geom}, expanding the geometric multiset by copying).

    We will show by induction that at any iteration $t$, there always exist injective functions $\varphi_s$ and $\varphi_v$ such that $\vs_i^{(t)} = \varphi_s ( c_i^{(t)} )$ and $\vec{\vv}_i^{(t)} = \varphi_v ( c_i^{(t)}, \vg_i^{(t)} )$.
    This holds for $t = 0$ because the initial node features are the same for GWL and GNN, $c_i^{(0)} \equiv \vs_i^{(0)}$ and $\vg_i^{(0)} \equiv (\vs_i^{(0)}, \vec{\vv}_i^{(0)})$ for all $i \in \mathcal{V}(\mathcal{G}), \mathcal{V}(\mathcal{H})$.
    Suppose this holds for iteration $t$.
    At iteration $t+1$, substituting $\vs_i^{(t)}$ with $\varphi_s(c_i^{(t)})$, and $\vec{\vv}_i^{(t)}$ with $\varphi_v(c_i^{(t)}, \vg_i^{(t)})$ gives us
    \begin{align*}
        \vs_i^{(t+1)}       & = \textsc{Upd}_s \left( \left[ (\varphi_s(c_i^{(t)}), \varphi_v(c_i^{(t)}, \vg_i^{(t)})) \ , \ \textsc{Agg} \left( \ldblbrace (\varphi_s(c_i^{(t)}), \varphi_s(c_j^{(t)}), \varphi_v(c_i^{(t)}, \vg_i^{(t)}), \varphi_v(c_j^{(t)}, \vg_j^{(t)})) \mid j \in \mathcal{N}_i \rdblbrace \right) \right] \right) \\
        \vec{\vv}_i^{(t+1)} & = \textsc{Upd}_v \left( (\varphi_s(c_i^{(t)}), \varphi_v(c_i^{(t)}, \vg_i^{(t)})) \ , \ \textsc{Agg} \left( \ldblbrace (\varphi_s(c_i^{(t)}), \varphi_s(c_j^{(t)}), \varphi_v(c_i^{(t)}, \vg_i^{(t)}), \varphi_v(c_j^{(t)}, \vg_j^{(t)})) \mid j \in \mathcal{N}_i \rdblbrace \right) \right)
    \end{align*}
    The composition of multiple injective functions is injective.
    Therefore, there exist some injective functions $g_s$ and $g_v$ such that:
    \begin{align}
        \vs_i^{(t+1)}       & = g_{s} \left( \left[ (c_i^{(t)}, \vg_i^{(t)}) \ , \ \ldblbrace (c_j^{(t)}, \vg_j^{(t)}) \mid j \in \mathcal{N}_i \rdblbrace \right] \right), \\
        \vec{\vv}_i^{(t+1)} & = g_{v} \left( (c_i^{(t)}, \vg_i^{(t)}) \ , \ \ldblbrace (c_j^{(t)}, \vg_j^{(t)}) \mid j \in \mathcal{N}_i \rdblbrace \right),
    \end{align}
    We can then consider:
    \begin{align}
        \vs_i^{(t+1)}       & = g_{s} \circ h_{s}^{-1} \ h_{s} \left( \left[ (c_i^{(t)}, \vg_i^{(t)}) \ , \ \ldblbrace (c_j^{(t)}, \vg_j^{(t)}) \mid j \in \mathcal{N}_i \rdblbrace \right] \right), \\
        \vec{\vv}_i^{(t+1)} & = g_{v} \circ h_{v}^{-1} \ h_{v} \left( (c_i^{(t)}, \vg_i^{(t)}) \ , \ \ldblbrace (c_j^{(t)}, \vg_j^{(t)}) \mid j \in \mathcal{N}_i \rdblbrace \right),
    \end{align}
    Then, we can denote $\varphi_s = g_s \circ h_{s}^{-1}$ and $\varphi_v = g_v \circ h_{v}^{-1}$ as injective functions because the composition of injective functions is injective.
    Hence, for any iteration $t+1$, there exist injective functions $\varphi_{s}$ and $\varphi_{v}$ such that $\vs_i^{(t+1)} = \varphi_{s} \left( c_i^{(t+1)} \right)$ and $\vec{\vv}_i^{(t+1)} = \varphi_{v} \left( c_i^{(t+1)} , \vg_i^{(t+1)} \right)$.

    At the $T$-th iteration, the GWL test decides that $\mathcal{G}$ and $\mathcal{H}$ are non-isomorphic, which means the multisets of node colours $\lbrace c_i^{(T)} \rbrace$ are different for $\mathcal{G}$ and $\mathcal{H}$.
    The GNN's node scalar features $\lbrace \vs_i^{(T)} \rbrace = \lbrace \varphi_s ( c_i^{(T)} ) \rbrace$ must also be different for $\mathcal{G}$ and $\mathcal{H}$ because of the injectivity of $\varphi_s$.

\end{proof}


A weaker set of conditions is sufficient for a $\fG$-invariant GNN to be at least as expressive as IGWL. 
\begin{restatable}{proposition}{invconditionstheo}
    \label{theo:invconditions}
    $\fG$-invariant GNNs have the same expressive power as IGWL if the following conditions hold:
    (1) The aggregation $\psi$ and update $\phi$ are $\fG$-orbit injective, $\fG$-invariant multiset functions.
    (2) The graph-level readout $f$ is an injective multiset function.
\end{restatable}
\begin{proof}
The proof follows similarly to the proof for Theorem~\ref{theo:conditions}. 
\end{proof}


\section{Universality and Discrimination Proofs (Section \ref{sec:universality})}
\label{app:universality}

\subsection{Equivalence between universality and discrimination}

The results in this subsection use the proofs from \citet{chen2019equivalence} with minor adaptations.

\FiniteUnivImpliesDiscrim*
\begin{proof}
    Given any $y \in Y$, we can construct the $\fG$-invariant function over $X$, $\delta_{y}(x) = 0$ if $y \simeq x$ and $1$ otherwise. Therefore, $\delta_{y}$ can be approximated with some $\eps < 0.5$ over $Y$ by some function $h \in \gC$. Hence, $h(y) \neq h(y')$ for any $y,y' \in Y$ and $\gC$ is pairwise $Y_\fG$ discriminating.
\end{proof}

The following two Lemmas follow from \citet{chen2019equivalence} with minor adaptations.

\begin{lemma}
    If $\gC$ is pairwise $Y_\fG$ discriminating, then for all $y \in Y$, there exists a function $\delta_y \in \gC^{+1}$ such that for all $y'$, $\delta_y(y') = 0$ if and only if $y \simeq y'$.
\end{lemma}

\begin{proof}
    For any $y_1, y_2 \in Y$ such that $y_1 \not\simeq y_2$, let $\delta_{y_1, y_2}$ be the function that distinguishes $y_1, y_2$. That is $\delta_{y_1, y_2}(y_1) \neq \delta_{y_1, y_2}(y_2)$. Then, we can define a function $\overline{\delta}_{y,y'} \in \gC$:
    \begin{align}
        \overline{\delta}_{y,y'}(x) & = |\delta_{y,y'}(x) - \delta_{y,y'}(y)| \rightarrow
        \begin{cases*}
            = 0    & if $x \simeq y$  \\
            > 0    & if $x \simeq y'$ \\
            \geq 0 & otherwise
        \end{cases*}
    \end{align}
    This function is already similar to the $\delta_y$ function whose existence we want to prove. To obtain a function that is strictly positive over all the $x \in Y$ with $x \not\simeq y$, we can construct $\delta_y$ as a sum over all the $\overline{\delta}_{y,y'}$:
    \begin{align}
        \delta_y(x) = \sum_{y' \in Y, y' \not\simeq y} \overline{\delta}_{y,y'}(x) \rightarrow
        \begin{cases*}
            = 0    & if $x \simeq y$                                          \\
            > 0    & if $x \not\simeq y$ and $x \in \gO_{\fG}(Y) \supseteq Y$ \\
            \geq 0 & otherwise
        \end{cases*}
    \end{align}
    Given the finite set of functions $\{ \delta_{y,y'} \}$, notice that
    $$\overline{\delta}_{y,y'}(x) = \textrm{ReLU}\big(\delta_{y,y'}(x) - \delta_{y,y'}(y)\big)\ +\ \textrm{ReLU}\big(\delta_{y,y'}(y) - \delta_{y,y'}(x)\big).$$
    Then $\delta_y$ is obtained by summing all these functions over $y' \in Y$ with $y' \not\simeq y$, so $\delta_y \in \gC^{+1}$.
\end{proof}

\begin{lemma}
    Let $\gC$ is a class of $\fG$-invariant functions from $X \to \sR$ such that for any $y, y' \in Y \subseteq X$, where $Y$ is finite, there is a $\delta_y \in \gC$ with the property $\delta_y(y') = 0$ if and only if $y \simeq y'$. Then $\gC^{+1}$ is universally approximating over $Y$.
\end{lemma}

\begin{proof}
    For $y \in Y$, define $r_y := \frac{1}{2}\min_{y' \in Y, y' \not\simeq y} \delta_y(y')$. Define the bump function with radius $r > 0$, $b_{r}: \sR \to \sR$ as $b_{r}(s)=\psi(\frac{s}{r})$, where
    $$\psi(z) = \textrm{ReLU}(z+1) + \textrm{ReLU}(1-z) - 2\textrm{ReLU}(z).$$
    Define $k_y := |Y \cap \gO_\fG(y)|^{-1}$. Since $Y$ is finite and the intersection with the orbit of $y$ contains $y$, $k_y$ is finite and well-defined. We can define the $\fG$-invariant function $h$ from $X$ to $\sR$ as:
    \begin{equation}\
        h(x) = \sum_{y \in Y} k_y f(y)b_{r_y}(\delta_y(x))
    \end{equation}
    Notice that $h|_Y = f|_Y$ and $h \in \gC^{+1}$. Therefore, $\gC^{+1}$ is universally approximating.
\end{proof}

\FiniteDiscrimImpliesUniv*

\begin{proof}
    Result follows directly from the two Lemmas above. 
\end{proof}

\begin{lemma}\label{lemma:cont_inf}
    Let $X, Y$ be topological spaces and $h: X \times Y \to \sR$ a continuous function. Then, if $Y$ is compact, $f(x) = \inf_{y \in Y} h(x, y)$ is continuous.
\end{lemma}

\begin{proof}
    The open sets $(-\infty, a)$ and $(b, \infty)$ form a basis for the topology of $\sR$. Thus, we show that their preimage under $f$ is open. First, notice $x \in f^{-1}((-\infty, a))$ if and only if $(x, y) \in h^{-1}((-\infty, a))$ for some $y \in Y$. Therefore, $f^{-1}((-\infty, a)) = p_X(h^{-1}((-\infty, a)))$, where $p_X: X \times Y \to X$ is the function projecting in the first argument. Since $p_X$  is continuous and open, it follows $p_X(h^{-1}((-\infty, a)))$ is open.

    When $x \in f^{-1}((b, \infty))$, it implies that for all $y \in Y, h(x, y) > b$. This means that for all $x \in f^{-1}((b, \infty))$ and $y \in Y$, we have $(x, y) \in h^{-1}((b, \infty))$. Since $h^{-1}((b, \infty)$ is open, then there exists an open box $U_{x,y} \times V_{x,y} \subseteq h^{-1}((b, \infty)$ containing $(x, y)$. Then, the union $\cup_{y \in Y} \big(U_{x,y} \times V_{x,y}\big)$ covers $\{x\} \times Y$. Since $Y$ is compact, there exists a finite subcover $\cup_{y_k}^{K_x} \big(U_{x,y_k} \times V_{x,y_k}\big)$
    of size $K_x$. Then notice that the open set $A_x := \cap_{y_k}^{K_x} U_{x, y_k}$ is a neighbourhood around $x$ and $A_x \times Y \subseteq h^{-1}((b, \infty))$. Therefore, $A_x \subseteq f^{-1}((b, \infty))$ and since $x \in f^{-1}((b, \infty))$ was chosen arbitrarily, $f^{-1}((b, \infty))$ is open.
\end{proof}

\UniversalCanSeparate*
\begin{proof}
    Consider $y, y' \in Y$ such that $y \not\simeq y'$. Then, the function $\delta_y(x) = \inf_{\fg \in \fG} d(y, \fg x) = \min_{\fg \in \fG} d(y, \fg x) > 0$, where the second equality follows from the compactness of $\fG$. This function is $\fG$-invariant. To show that it is continuous, notice that the function $h(x, \fg) = d(y, \fg x)$ is given by the composition $d_y \circ a$, where $a: X \times \fG \to X$ is the continuous group action and $d_y: X \to \sR$ is given by $d_y(x) = d(y, x)$, which is also continuous. Since composition of continuous functions is continuous and $\delta_y(x) = \inf_{\fg \in \fG} h(x, \fG)$, it follows from Lemma \ref{lemma:cont_inf} that $\delta_y$ is a continuous function.

    Given a universally approximating class of functions $\gC$, we can find a function $f$ approximating $\delta_y$ with precision $\eps < \frac{\delta_y(y')}{2}$ and, therefore, $f(y) \neq f(y')$.
\end{proof}

\begin{definition}\label{def:locate_orbit_function}
    Let $\gC$ be a class of functions $X \to \sR$ and $Y \subseteq X$. We say that $\gC$ can locate every orbit over $Y$ if for any $y \in Y$ and any $\varepsilon > 0$ there exists $\delta_y \in \gC$ such that:
    \begin{enumerate}
        \item For all $y' \in Y, \delta_y(y') \geq 0$.
        \item For all $y' \in Y$, if $y \simeq y'$, then $\delta_y(y') = 0$.
        \item There exists $r_y > 0$ such that if $\delta_y(y') < r_y$ for any $y' \in Y$, then there is a $\fg \in \fG$ such that $d(y', \fg \cdot y) < \varepsilon$.
    \end{enumerate}
\end{definition}

Notice that since $\delta_y \in \gC$, it is $\fG$-invariant and then for any $y^* \in \gO_\fG(y')$, $\delta_y(y') = \delta_y(y^*)$ and there exists $\fg \in \fG$ such that $d(y^*, \fg \cdot y) < \varepsilon$. Therefore, intuitively one should see $\delta_y$ as some sort of ``distance function'' measuring how far all $y^* \in \gO_\fG(y')$ are from the orbit of $y$. In other words, when $\delta_y(y^*)$ is low, it means that the entire orbit of $y^*$ is close to the orbit of $y$.

\begin{lemma}
    If $\gC$ is a collection of continuous $\fG$-invariant functions from $X \to \sR$ that is pairwise $Y_\fG$-discriminating, then $\gC^{+1}$ is able to locate every orbit over $Y$.
\end{lemma}

\begin{proof}
    Select an arbitrary $y \in Y$. For $y' \not\simeq y$, let $\delta_{y,y'}$ be the function in $\gC$ separating $y$ and $y'$. Consider the radius $r_{y, y'} := \frac{1}{2}|\delta_{y, y'}(y) - \delta_{y, y'}(y')| > 0$ and define the set
    $$A_{y'} := \delta_{y, y'}^{-1}\left( \delta_{y, y'}(y') - r_y, \delta_{y, y'}(y') + r_y \right)$$
    Since $\delta_{y, y'}$ is continuous, $A_{y'}$ is open. If $y' \simeq y$, then we define $A_y = B(y, \varepsilon)$, the open ball in $X$ centred at $y$ with radius $\varepsilon$. Clearly, $\bigcup_{y' \in Y} A_{y'}$ forms a cover for $Y$. Since $Y$ is compact, then there exists a finite subcover given by a finite subset $Y_0 \subseteq Y$ such that $\bigcup_{y' \in Y_0} A_{y'}$.

    We construct the function $\delta_y(y')$ over $X$ as $\delta_y(y') := \sum_{y' \in Y_0 \setminus \gO_\fG(y^*)} \bar{\delta}_{y, y'}(y^*)$, where $\bar{\delta}_{y, y'}(y^*) := \max(\frac{4}{3} r_{y, y'} - |\delta_{y, y'}(y^*) - \delta_{y, y'}(y')|, 0)$. Since $\delta_{y, y'}$ is continuous and $\fG$-invariant, so is $\delta_y$. Finally, it can be shown that $\delta_y$ can indeed locate the orbit of $y$ over $Y$.

    \begin{enumerate}
        \item Clearly, $\delta_y(x) \geq 0$ for any $x \in X$.
        \item For any $y^* \in Y$, if $y \simeq y^*$, then because $\delta_{y, y'}$ is $\fG$-invariant, we have $\delta_{y, y'}(y^*) = \delta_{y, y'}(y)$, so $\delta_y(y^*) = 0$.
        \item First, consider $y^* \in Y$ such that $\forall \fg \in \fG$, $d(\fg \cdot y^*, y) \geq \varepsilon$. Then, $y^* \in Y \setminus \bigcup_{y' \in \gO_\fG(y)} A_{y'}$. Then, there must be a $y' \in Y_0 \setminus \gO_\fG(y)$, such that $y^* \in A_{y'}$. Therefore, $|\delta_{y, y'}(y^*) - \delta_{y, y'}(y')| < r_{y,y'} < \frac{4}{3} r_{y,y'}$. Then, we have $\frac{4}{3}r_{y, y'} - |\delta_{y, y'}(y^*) - \delta_{y, y'}(y')| > \frac{4}{3}r_{y, y'} - r_{y, y'} = \frac{1}{3}r_{y, y'} > 0 \Rightarrow \bar{\delta}_{y, y'}(y^*) > \frac{1}{3}r_{y, y'}$. Therefore, we can set $r_y := \frac{1}{3}\min_{y' \in Y_0 \setminus \gO_\fG(y)} r_{y, y'} > 0$. If $\delta_y(y^*) < r_y$ it follows that for all $y' \in Y_0 \setminus \gO_\fG(y^*)$, $\bar{\delta}_{y, y'}(y^*) < \frac{1}{3}r_{y, y'}$, which implies $y^* \in \bigcup_{\fg \in \fG} B(\fg \cdot y, \varepsilon)$. Finally, this proves there is a $\fg \in \fG$ such that $d(y^*, \fg \cdot y) < \varepsilon$.
    \end{enumerate}

    Since the absolute value function can be realised using ReLU activations, it is easy to see that $\delta_y \in \gC^{+1}$.
\end{proof}

\begin{lemma}
    If $\gC$, a class of functions over a compact set $Y$, can locate every isomorphism class, then $\gC^{+2}$ is universal approximating over $Y$.
\end{lemma}

\begin{proof}
    Consider any continuous and $\fG$-invariant function on $X$. Since $Y \subseteq X$ is compact, then $f$ is uniformly continuous when restricted to $Y$. In other words, for all $\varepsilon > 0$, there exists $r > 0$, such that for all $y_1, y_2 \in Y$, if $d(y_1, y_2) < r$, then $|f(y_1) - f(y_2)| < \varepsilon$.

    Let $y \in Y$ and define $B_\fG(y, r) := \bigcup_{y' \in \gO_\fG(y)} B(y', r)$ to be the union of all the open balls of radius $r$ on the orbit of $y$. Using the function $\delta_y$ from Definition \ref{def:locate_orbit_function}, there exists $r_y$ such that $\delta_y^{-1}([0, r_y)) \subseteq B_\fG(y, r)$ for any $y \in Y$. Since $\delta_y$ is continuous, $\delta_y^{-1}([0, r_y))$ is open. Therefore, $\{ \delta_y^{-1}([0, r_y)) \}_{y \in Y}$ is an open cover for $Y$. Since $Y$ is compact, we can find a finite subcover $\{ \delta_y^{-1}([0, r_y)) \}_{y \in Y_0}$, where $Y_0$ is a finite subset of $Y$.

    We can now use the functions $\delta_y$ to construct a set of continuous $\fG$-invariant functions that forms a partition of unity for this finite cover. For $y_0 \in Y_0$ we construct the function $\phi_{y_0}(y') = \max(r_{y_0} - \delta_{y_0}(y'), 0)$ and the function $\phi(y') = \sum_{y^* \in Y_0} \phi_{y^*}(y')$, both of which are continuous. Noticing that $\mathrm{supp}(\phi_{y_0}) = \delta_{y_0}^{-1}([0, r_{y_0}))$ and that $\phi_{y^*}(y') > 0$ for any $y' \in Y$, the set of functions $\psi_{y_0}(y') = \frac{\phi_{y_0}(y')}{\phi(y')}$ form a partition of unity with $\sum_{y_0 \in Y_0} \psi_{y_0}(y') = 1$ for all $y' \in Y$.

    Notice that we can write any $\fG$-invariant function $f$ as:
    \begin{equation}
        f(y') =  f(y') \sum_{y_0 \in Y_0} \psi_{y_0}(y') =  \sum_{y_0 \in Y_0\ :\ y' \in \delta_{y_0}^{-1}([0, r_{y_0}))} f(y')\psi_{y_0}(y')
    \end{equation}
    The intuition is that because $f$ is continuous, we can approximate $f(y')$ in the expression above by the value of $f(y_0)$ since $y'$ is in the neighbourhood of some $y_0$. Thus, the function that approximates $f$ is $h(y') = \sum_{y_0 \in Y_0} f(y_0)\psi_{y_0}(y')$. We now show that $h$ can approximate $f$ with arbitrary accuracy.

    If $y' \in \delta_{y_0}^{-1}([0, r_{y_0}))$, then there exists $\fg \in \fG$ such that $d(y', \fg \cdot y_0) < r$. Using the fact that $f$ is continous, this implies $|f(y') - f(\fg \cdot y_0)| < \varepsilon$. Because $f$ is invariant, $f(y_0) = f(\fg \cdot y_0)$, which implies $|f(y') - f(y_0)| < \varepsilon$. Then we have:
    \begin{align}
        \big|f(y') - \sum_{y_0 \in Y_0} f(y_0)\psi_{y_0}(y') \big| & = \big| f(y') - \sum_{y_0 \in Y_0\ :\ y' \in \delta_{y_0}^{-1}([0, r_{y_0}))} f(y_0)\psi_{y_0}(y') \big|                   \\
                                                                   & = \sum_{y_0 \in Y_0\ :\ y' \in \delta_{y_0}^{-1}([0, r_{y_0}))} \left| f(y') - f(y_0) \right| \psi_{y_0}(y') < \varepsilon
    \end{align}
    Finally, to see that $h$ is in $\gC+2$, we can use an MLP with one hidden layer to approximate $\psi_{y_0}$ followed by one final layer to compute the linear combination of the $\psi_{y_0}$.
\end{proof}

\SeparationImpliesUnivers*
\begin{proof}
    The proof follows from the two Lemmas above.
\end{proof}

\subsection{Number of aggregators in continuous setting}

\NumAgg*
\begin{proof}
    Suppose for the sake of contradiction that there exists an orbit-space injective and continuous function $f: X \to \sR^m$, with $m<d$. Since $Y$ is a submanifold of the same dimension as $X$, then $f$ must also be injective over $Y$. 
    By the Quotient Manifold Theorem \citep{lee2013smooth}, $Y/\fG$ is a topological manifold of dimension $d = \mathrm{dim\ } X - \mathrm{dim\ } \fG$.
    The map $f$ induces an injective function $g: Y/\fG \to \sR^m$. This map is also continuous because for an open set $V \in \sR^m$, $g^{-1}(V) = \pi_Y(f^{-1}(V))$. Because $f$ is continuous and $\pi_Y$ is an open map, this set is open.
    \begin{center}
        \begin{tikzcd}
            & Y \arrow[d, "\pi_Y"] \arrow[dr, "f"] \\
            \sR^d \arrow[r, "\psi^{-1}"] & Y/\fG \arrow[r, "g"] & \sR^m
        \end{tikzcd}
    \end{center}
    Because $Y/\fG$ is a manifold, there exist an open set $U \subseteq Y/\fG$ and a homeomorphism $\psi: U \to \sR^d$. Then the composition $h = g \circ \psi^{-1}$ is a continuous and injective map from $\psi(U) \subseteq \sR^m$ to $\sR^m$. By the Invariance of Domain Theorem~\cite{bredon2013topology} (Corollary 19.9), $h$ is open and it is a homeomorphism onto its image $h(\psi(U)) \subseteq \sR^m$. By the Invariance of Dimension Theorem~\cite{bredon2013topology} (Corollary 19.10), $d = m$.
\end{proof}

\SOdAgg*
\begin{proof}
    We now consider the case when $\fG = S_n \times SO(d)$. First, notice that the proof above also holds for this group since it is a subgroup of $S_n \times O(d)$. However, we can obtain a stronger result and show the result holds for $n \geq d - 1$. In what follows, we reuse the notation from the proof above.

    We define the set
    $$Z' = \{ \mX \in X \mid \exists\ 1 \leq i_1 < \ldots < i_{d-1} \leq n \textrm{\ s.t\ }  \vx_{i_1}, \ldots, \vx_{i_{d-1}} \textrm{\ are linearly independent\ }\},$$
    containing $d-1$ row-vectors that are linearly independent. Define $M_\mX$ to be the set of all $(d-1) \times (d-1)$ minors of the matrix $\mX$. Then, we can construct a continuous function $h(\mX) = \max_{m \in M_\mX} |m|$ and notice that $Z'$ coincides with the open set $h^{-1}((0, \infty))$. Then, the set $V = Y \cap Z'$ is also open and non-empty. Therefore, $V$ is a submanifold of $X$ of the same dimension and the action of $\fG$ is well-defined and continuous on $W$. We can show again this action is free.

    As in the proof above, we have that $\mP_\fg = \mI_n$, so it remains to inspect the case $\mX = \mX \mQ_\fg$. Any non-trivial rotation $\mQ_\fg$ must rotate at least a two-dimensional subspace of $\sR^d$. Since the rows of the matrix $\mX$ span a $(d-1)$-dimensional subspace of $\sR^d$, then $\mQ_\fg$ cannot leave $\mX$ invariant unless $\mQ_\fg = \mI_d$. Applying Theorem \ref{theo:num_agg} again yields the result.

    For $n = d = 1$, $\norm{\cdot}: \sR^{1 \times d} \to \sR$ is, as before, $\fG$ orbit-space injective.
\end{proof}

\OdAgg*
\begin{proof}

    First, suppose that $n \geq d > 1$. Consider the subspace $Y = \{ \mX \in X \mid \norm{\vx_i} \neq \norm{\vx_j}, \forall i < j\}$, where the norm is just the standard Euclidean norm. Consider the function $g: X \to \sR$ given by $g(\mX) = \min_{i < j} |\norm{\vx_i} - \norm{\vx_j}|$. By standard analysis, this function is continuous and notice that $Y = g^{-1}((0, \infty))$, which means that $Y$ is open in $X$. We also define the set
    $$Z = \{ \mX \in X \mid \exists\ i_1 < \ldots < i_d < n, |\mathrm{det}(\vx_{i_1}, \ldots, \vx_{i_d})| > 0\},$$
    containing row-vectors that span $\sR^d$. As above, this set is the preimage of the absolute determinant over $(0, \infty)$, which makes $Z$ open in $X$. Then, the set $W = Y \cap Z$ is also open and non-empty. Therefore, $W$ is a submanifold of $X$ of the same dimension and the action of $\fG$ is well-defined and continuous on $W$. We can show this action is free.

    We investigate the solutions of the equation $\mP_\fg \mX \mQ_\fg^\top = \mX \iff \mP_\fg \mX = \mX \mQ_\fg$ for $\mX \in W$. Since orthogonal transformations preserve norms and the rows of $\mX$ have different norms, it follows that $\mP_\fg = \mI_n \Rightarrow \mX = \mX \mQ_\fg$. We know that a subset of the rows of $\mX$ span the whole of $\sR^d$. Define the sub-matrix of $\mX$ containing these rows by $\mX^* \in \sR^{d \times d}$. Then, we have $\mX^*\mQ_\fg = \mX^* \Rightarrow \mQ_\fg = (\mX^*)^{-1}\mX^* = \mI_d$. This proves that the action is free and applying Theorem \ref{theo:num_agg}, yields the result.

    For the trivial case when $n = 1$, notice that $\norm{\cdot}: \sR^{1 \times d} \to \sR$ is $\fG$ orbit-space injective.
\end{proof}

\generalisedPNA*
\begin{proof}
    Reusing the notation from above, notice that for all $n \geq 1$, $S_n$ acts freely on the sub-manifold $Y$ as shown above. Seeing $S_n$ as a zero-dimensional Lie group and applying Theorem \ref{theo:num_agg} yields the result.
\end{proof}



\end{document}